\documentclass[11pt]{article}

\usepackage[margin=1.02in]{geometry}
\usepackage[T1]{fontenc}
\usepackage{lmodern}
\usepackage{microtype}
\usepackage{amsmath,amssymb,amsfonts,amsthm,mathtools,bm}
\usepackage{booktabs,tabularx,array}
\usepackage{graphicx}
\usepackage{enumitem}
\usepackage{natbib}
\usepackage[hidelinks]{hyperref}
\usepackage[nameinlink,noabbrev]{cleveref}
\usepackage{xcolor}
\usepackage{setspace}
\usepackage{appendix}
\usepackage{placeins}

\graphicspath{{figures/}}
\allowdisplaybreaks
\newtheorem{theorem}{Theorem}[section]
\newtheorem{proposition}[theorem]{Proposition}
\newtheorem{corollary}[theorem]{Corollary}
\newtheorem{lemma}[theorem]{Lemma}
\newtheorem{assumption}[theorem]{Assumption}
\newtheorem{definition}[theorem]{Definition}
\newtheorem{remark}[theorem]{Remark}

\newcommand{\R}{\mathbb{R}}
\newcommand{\E}{\mathbb{E}}
\newcommand{\Pp}{\mathbb{P}}
\newcommand{\cF}{\mathcal{F}}

\newcommand{\cX}{\mathcal{X}}

\newcommand{\ip}[2]{\left\langle #1,#2\right\rangle}
\newcommand{\norm}[1]{\left\lVert #1\right\rVert}
\newcommand{\abs}[1]{\left\lvert #1\right\rvert}
\newcommand{\Cov}{\operatorname{Cov}}
\newcommand{\Var}{\operatorname{Var}}
\newcommand{\diag}{\operatorname{diag}}
\newcommand{\tr}{\operatorname{tr}}

\newcommand{\as}{\mathrm{a.s.}}
\newcommand{\trans}{\mathsf T}
\newcommand{\eps}{\varepsilon}
\newcommand{\NAG}{\mathrm{NAG}}

\title{\textbf{Geometric Moment Contraction for Stochastic Nesterov Acceleration:}\\
Stationarity, Heavy Tails, and Initialization-Robust Limit Theory}
\author{Wei Biao Wu}
\date{Revised manuscript -- September 25, 2026}

\hypersetup{
  pdftitle={Geometric Moment Contraction for Stochastic Nesterov Acceleration},
  pdfauthor={Wei Biao Wu},
  pdfsubject={Stationarity and limit theory for constant-parameter stochastic Nesterov acceleration},
  pdfkeywords={Nesterov acceleration, stochastic gradient descent, geometric moment contraction, stationary distribution, heavy tails, invariance principle}
}

\begin{document}
\maketitle

\begin{abstract}
We study geometric moment contraction (GMC) of the constant-parameter stochastic Nesterov recursion
\[
Y_k=\Theta_k+\beta(\Theta_k-\Theta_{k-1}),\qquad
\Theta_{k+1}=Y_k-\gamma G(Y_k,X_{k+1}).
\]
Under mean strong monotonicity and stochastic $L^p$ Lipschitz continuity, an explicit Perron comparison proves synchronous $L^p$ contraction when
$\beta\gamma L_p<(1-\beta)(1-q_{\gamma,p})$.
This direct criterion includes infinite-variance gradients for $1<p<2$, but its small-step regime requires $\beta<\mu/(\mu+L_p)$.  A complementary power-Lyapunov argument establishes a positive, generally much smaller, step-size interval for every fixed $\beta<1$ and every $p>1$, using only a finite $p$th gradient moment.  At $p=2$, a simpler explicit certificate gives
\[
0<\gamma<\frac{2\mu(1-\beta)^2}{L_2^2(1-\beta+2\beta^2)}.
\]
Its quadratic high-momentum scaling is a limitation of the chosen metric, not a sharp stability boundary.  We quantify this loss, provide a general mean-only quadratic $S$-procedure, and exploit endpoint Lyapunov inequalities under stronger samplewise sector information.  Verified endpoint certificates can be orders of magnitude less conservative than the explicit metric.

The resulting GMC bounds yield a causal stationary law, exponential physical dependence, and limit theory for the full observed orbit from every deterministic initial pair.  We state moment and almost-sure rates, a central limit theorem, and an $o_P(n^{1/p})$ Gaussian approximation under the corresponding order-$p>2$ contraction.  An explicit stationary second-moment bound has scale $\gamma/(1-\beta)$ in its stated small-step regime.  Exact additive-quadratic formulas give the Schur region, stationary variance, momentum-invariant long-run covariance, and stable limits.  An exact scalar random-curvature benchmark exhibits constant, square-root, and linear high-momentum step-size scales, separated by $L_2^2=3\mu^2$.  Numerical experiments distinguish valid certificates from observed rates and cover random curvature, infinite-variance noise, initialization effects, and long-run variance.
\end{abstract}

\noindent\textbf{Keywords:} stochastic Nesterov acceleration; geometric moment contraction; iterated random functions; stationary distribution; heavy-tailed gradients; physical dependence; initialization-robust central limit theorem; Gaussian approximation.

\noindent\textbf{MSC 2020:} 60G10; 60F05; 60F17; 62L20; 90C15.

\section{Introduction}
\label{sec:intro}

Nesterov's accelerated gradient method is one of the foundational constructions in first-order optimization.  In the deterministic smooth convex setting, extrapolation changes the worst-case iteration complexity from the gradient-descent scale to the accelerated scale \citep{Nesterov1983,Nesterov2004}.  The method has inspired estimate-sequence analyses, continuous-time limits, variational formulations, and control-theoretic Lyapunov certificates \citep{SuBoydCandes2016,WibisonoWilsonJordan2016,LessardRechtPackard2016,HuLessard2017}.  Stochastic versions are widely used, but their behavior differs sharply from their deterministic counterpart: persistent gradient noise creates a nondegenerate steady state, acceleration can amplify variance, and parameter values that are benign for ordinary SGD may destabilize Nesterov's method \citep{CohenDiakonikolasOrecchia2018,CanGurbuzbalabanZhu2019,AssranRabbat2020,LiuBelkin2020,GaneshEtAl2023,AttiaKoren2021}.

We study the constant-parameter recursion
\begin{equation}
\begin{split}
  Y_k&=\Theta_k+\beta(\Theta_k-\Theta_{k-1}),\\
  \Theta_{k+1}&=Y_k-\gamma G(Y_k,X_{k+1}),
\end{split}
\qquad k\geq0,
\label{eq:nag}
\end{equation}
where $0\leq\beta<1$, $\gamma>0$, $G(\theta,x)=\nabla_\theta g(\theta,x)$, and $(X_k)$ is an i.i.d. data stream.  This is the standard look-ahead form analyzed, for example, by \citet{AssranRabbat2020}.  Some alternative indexing and velocity conventions are linearly equivalent to \eqref{eq:nag} for constant parameters.  We study this specific look-ahead recursion; results are not asserted for every algorithm marketed as Nesterov momentum.

With constant $\gamma$, the iterates generally do not converge to the optimizer.  The natural object is instead the invariant law of the augmented state $(\Theta_k,\Theta_{k-1})$.  Existing work gives important optimization-error, Wasserstein, local-stability, and decreasing-step asymptotics.  \citet{CanGurbuzbalabanZhu2019} prove convergence toward a unique invariant distribution in Wasserstein distance under structured stochastic-oracle conditions; \citet{AssranRabbat2020} analyze the stochastic steady-state neighborhood; \citet{GitmanLangZhangXiao2019} characterize local stationary behavior in a unified momentum family; and \citet{BarakatEtAl2021} study convergence, fluctuations, and trap avoidance for decreasing-step momentum algorithms, including stochastic Nesterov acceleration.  Stochastic Lyapunov and robust acceleration results include \citet{LabordeOberman2020,AybatFallahGurbuzbalabanOzdaglar2019,AybatFallahGurbuzbalabanOzdaglar2020,GuptaSiegelWojtowytsch2024,YuChenFeng2026}.  Our focus is a complementary, explicitly verified synchronous $L^p$ coupling for the nonlinear augmented state, together with the statistical consequences of its dependence structure.

Our starting point is the geometric moment contraction (GMC) theory for ordinary constant-step SGD.  If
\[
 F_x(\theta)=\theta-\gamma G(\theta,x),
\]
recent work establishes explicit conditions under which
\begin{equation}
 \norm{F_X(\theta)-F_X(\vartheta)}_p
 \leq q_{\gamma,p}\abs{\theta-\vartheta},
 \qquad q_{\gamma,p}<1,
 \label{eq:intro-base}
\end{equation}
for every $p>1$ \citep{LiLouRichterWu2026}.  This is a coupling statement in Euclidean $L^p$, not merely convergence of marginal laws.  It yields a causal stationary process, exponential forgetting of deterministic initialization, physical-dependence bounds, and limit theory for the actually observed nonstationary orbit.  The case $1<p<2$ is particularly important because it permits infinite gradient variance.

Extending \eqref{eq:intro-base} to \eqref{eq:nag} is not formal.  Nesterov evaluates the gradient at the extrapolated point $Y_k$ but carries the iterate $\Theta_k$ as a second state.  If two trajectories share the same innovations, their extrapolated-point and increment differences feed into each other.  Componentwise absolute values lose the negative-sign cancellation that stabilizes large momentum.  We therefore separate three constructions: a Perron-weighted product bound with a restrictive momentum range, a quadratic energy with a simple but conservative all-momentum mean-square region, and powers of that energy giving finite-$p$-moment contraction at arbitrary fixed momentum.  General quadratic certificates and exact linear benchmarks make the conservatism visible rather than attributing it entirely to stochastic curvature.

\subsection{Main contributions}

Let $q_{\gamma,p}$ be the ordinary-SGD coefficient in \eqref{eq:intro-base} and let $L_p$ be the stochastic $L^p$ Lipschitz constant of $G$.  The paper makes the following contributions.

\begin{enumerate}[label=\textup{(\roman*)},leftmargin=2.3em]
\item \textbf{Direct $L^p$ GMC, including infinite variance.}
For synchronous Nesterov trajectories, define the extrapolated-point difference $y_k$ and increment difference $s_k$.  Their $L^p$ norms satisfy a two-dimensional nonnegative comparison recursion with matrix
\begin{equation}
 M^{\NAG}_{\gamma,\beta,p}
 =\begin{pmatrix}
 q_{\gamma,p}+\beta\gamma L_p & \beta^2\\
 \gamma L_p & \beta
 \end{pmatrix}.
 \label{eq:intro-matrix}
\end{equation}
Its Perron root is strictly below one if and only if
\begin{equation}
 \beta\gamma L_p<(1-\beta)(1-q_{\gamma,p}).
 \label{eq:intro-perron}
\end{equation}
We give the exact Perron root of this \emph{comparison matrix}, not the exact contraction rate of the algorithm, and a positive left-Perron weight.  The result applies for every $p>1$, including $1<p<2$, but its small-step regime requires $\beta<\mu/(\mu+L_p)\leq1/2$.  The power-Lyapunov theorem in Section~\ref{sec:power-extension} removes this qualitative momentum restriction at the price of a substantially smaller explicit step-size interval.

\item \textbf{Mean-square GMC for every momentum value.}
For $p=2$, a transformed slow coordinate
\[
 u_k=x_k+\frac{\beta}{1-\beta}s_k
\]
obeys the exact identity $u_{k+1}=u_k-\gamma h_{k+1}/(1-\beta)$, where $h_{k+1}$ is the synchronous gradient difference at the look-ahead point.  The quadratic energy
\[
 V_k=\abs{u_k}^2+\frac{\beta^2}{(1-\beta)^2}\abs{s_k}^2
\]
then gives GMC for every $0\leq\beta<1$ under
\begin{equation}
 0<\gamma<
 \frac{2\mu(1-\beta)^2}
 {L_2^2\{1-\beta+2\beta^2\}}.
 \label{eq:intro-allbeta}
\end{equation}
Only the mean field is required to be strongly monotone; individual sample losses may be nonconvex.  The displayed weight optimizes the \emph{derived sufficient condition} within $\abs{u}^2+a\abs{s}^2$ with this slow coordinate held fixed.  It is not an optimum over all quadratic metrics, and the $(1-\beta)^2$ scaling is not intrinsic stability scaling.  The coordinate is the standard averaged-sequence transformation; the contribution is its use in a synchronous contraction metric and the explicit weight calculation.

\item \textbf{Pathwise and LMI refinements.}
Under samplewise Hessian-sector bounds, the same energy contracts pathwise and therefore in every $L^p$ for which the anchor has a finite moment.  A common quadratic Lyapunov inequality need only be checked at the two sector endpoints, by the standard polytopic vertex argument.  We implement the test and publish positive-definite matrices with checked residuals.  A separate $S$-procedure uses only mean monotonicity and second-moment Lipschitz information, without samplewise convexity.

\item \textbf{Stationarity and statistical limit theory.}
The contraction implies a unique causal stationary state, finite moments, exponential replacement-of-the-past GMC, and exponentially decaying physical dependence.  We transfer established causal-process moment and almost-sure theorems to averaged Nesterov iterates, obtain a central limit theorem valid from every deterministic pair $(\Theta_{-1},\Theta_0)$, and obtain an $o_P(n^{1/p})$ Gaussian approximation under an order-$p>2$ contraction.  A mean-square certificate alone does not imply the higher-moment coupling at the same step size.  We call these results \emph{initialization-robust} or \emph{initial-state quenched}: no random stationary initialization is imposed on the observed recursion.

\item \textbf{Stationary scale, invariant bias, and exact quadratic benchmarks.}
We make the momentum dependence explicit: in the specified small-step regime, the stationary second moment is bounded by a stated constant times $\gamma/(1-\beta)$, and the invariant-mean bias has the same bound under a globally Lipschitz Jacobian.  In particular, these orders are $O(\gamma)$ for fixed $\beta$.  In the additive quadratic model we derive the exact stability region, stationary variance, and long-run covariance.  The latter equals $H^{-1}\Omega H^{-1}$ and is independent of both $\gamma$ and $\beta$.  Under regularly varying innovations we also obtain an explicit stable limit for partial sums.  An exactly solvable random-curvature example gives a necessary upper envelope for any universal mean-only certificate and shows three distinct high-momentum regimes.
\end{enumerate}

\subsection{Positioning relative to the literature}

\paragraph{Momentum and acceleration.}
The deterministic heavy-ball and accelerated-gradient methods originate in \citet{Polyak1964,Nesterov1983}; standard treatments include \citet{Polyak1987,Nesterov2004}.  Momentum became central in large-scale learning after the empirical studies of \citet{SutskeverEtAl2013}, and a broad family of variants is now represented by stochastic heavy ball, Nesterov momentum, quasi-hyperbolic momentum, and unified momentum schemes \citep{YangLinLi2016,YanEtAl2018,MaYarats2019,GitmanLangZhangXiao2019,LoizouRichtarik2020,WangYurtsever2026}.  Deterministic analyses based on estimate sequences, ordinary differential equations, variational principles, memory models, and integral quadratic constraints include \citet{GhadimiFeyzmahdavianJohansson2015,SuBoydCandes2016,WibisonoWilsonJordan2016,LessardRechtPackard2016,HuLessard2017,OrvietoKohlerLucchi2020,PapazovEtAl2024}.  These works explain acceleration of optimization transients, whereas the present paper studies contraction and stationary probability for a permanently noisy, fixed-parameter recursion.

\paragraph{Stochastic Nesterov and momentum dynamics.}
Early two-step stochastic-approximation analysis goes back to \citet{Kaniovski1983}.  Modern results study optimization error, almost-sure convergence, delayed supermartingale arguments, diffusion limits, and local stationary behavior for heavy-ball or generalized momentum algorithms \citep{YuanYingSayed2016,GadatPanloupSaadane2018,LiuGaoYin2020,SebbouhGowerDefazio2021,JinEtAl2022,Zhang2024Delayed,LiuChenZhouZhao2021,FengJiangWangYing2023,BarakatEtAl2021,GessKassing2026,GallonJentzen2026}.  For Nesterov specifically, \citet{CanGurbuzbalabanZhu2019} prove convergence toward an invariant distribution in Wasserstein distance under a structured oracle model, and \citet{AssranRabbat2020} quantify convergence to a noise neighborhood and identify possible finite-sum divergence.  Robustness and possible failure of stochastic acceleration are studied in \citet{CohenDiakonikolasOrecchia2018,KidambiEtAl2018,AttiaKoren2021,GaneshEtAl2023}; acceleration under optimal stochastic-oracle designs, least-squares structure, interpolation, multiplicative noise, or specially designed robust schemes is investigated by \citet{Lan2012,GhadimiLan2013,KulunchakovMairal2019,JainKakadeKidambiNetrapalliSidford2018,LiuBelkin2020,EvenEtAl2021,GuptaSiegelWojtowytsch2024,YuChenFeng2026}.  The high-dimensional behavior of momentum is studied by \citet{JagannathJonesMcCormickSarangian2026}.  Recent generalized-momentum and algorithmic-stability theories encompass both Polyak and Nesterov schemes but focus on optimization and generalization rather than synchronous moment contraction \citep{RamezaniKebryaEtAl2024,WangYurtsever2026,LeiWangYuan2026}.

\paragraph{Constant-step stationarity and inference.}
Constant-step stochastic approximation has long been understood to fluctuate around its target rather than converge pointwise \citep{Pflug1986}; broad optimization and stochastic-approximation treatments include \citet{MoulinesBach2011,BottouCurtisNocedal2018}.  Markov-chain, weak-error, diffusion, and concentration approaches to the stationary law of SGD include \citet{DieuleveutDurmusBach2020,FengEtAl2020,ChenMouMaguluri2022,YuEtAl2021,LouZhuWu2022,MeradGaiffas2023}.  Statistical inference for averaged stochastic approximation and SGD builds on \citet{Fabian1968,Ruppert1988,PolyakJuditsky1992,ChenLeeTongZhang2020,MouEtAl2020,ZhuChenWu2023}.  CLTs for momentum and Nesterov-type algorithms have also been considered, commonly with decreasing gains or local asymptotics \citep{LiXiaoYang2023,TangLiuZhangChen2023,AnHuo2026}.  Our fixed-parameter result is different: the joint Nesterov state is first shown to be globally GMC, and the CLT and Gaussian approximation are then transferred from the causal stationary process to every deterministic initial pair.

\paragraph{Heavy tails.}
Heavy-tailed stationary laws of stochastic recursions and SGD are analyzed in \citet{Mirek2011,BuraczewskiDamekMirek2012,GurbuzbalabanSimsekliZhu2021,HodgkinsonMahoney2021,DamekMentemeier2026}.  Momentum under heavy-tailed noise has been studied through fractional underdamped dynamics, stability, and finite-time optimization bounds \citep{SimsekliZhuTehGurbuzbalaban2020,DangEtAl2025,YamadaSatoIiduka2026}.  Robust accelerated algorithms often clip or otherwise modify the gradients \citep{GorbunovDanilovaGasnikov2020}; recent inference theory for infinite-variance SGD uses stable and self-normalized limits \citep{BlanchetMijatovicYang2025,BlanchetGlynnYang2026}.  Theorem~\ref{thm:direct} proves contraction of the \emph{unmodified} Nesterov recursion from a finite $p$th moment, including $1<p<2$, with the direct criterion restricting momentum to $\beta<\mu/(\mu+L_p)$ in the small-step regime.  Theorem~\ref{thm:allp} supplies an all-fixed-momentum alternative under the same moment order, but its explicit interval is generally far smaller.  GMC controls memory; it does not by itself create regular variation, which is why the nonlinear stable-law discussion explicitly separates contraction from tail assumptions.

\paragraph{Nonlinear autoregressive stability.}
\citet{ChenWu2016} is a direct methodological predecessor: its Theorem~1, Corollary~6, and Theorems~5--6 connect stochastic Lipschitz lag coefficients to causal stationarity, functional dependence, and limit theory.  Section~\ref{sec:ar-connection} specializes that approach to Nesterov's nonlinear AR(2) representation and identifies what is gained by the alternative augmented-state metrics.  Thus the contraction-to-inference principle is prior work, not an innovation of the present paper.

\paragraph{Iterated random functions and nonlinear time series.}
The probabilistic foundation is the theory of iterated random functions and stochastic recursive systems \citep{DubinsFreedman1966,Duflo1997,BarnsleyEltonHardin1989,DiaconisFreedman1999,Stenflo2012}.  GMC and physical dependence were developed by \citet{WuShao2004,Wu2005}; the moment inequalities, quenched invariance principles, strong invariance principles, and multivariate Gaussian couplings used below come from \citet{CunyMerlevede2014,Wu2007,Wu2011,BerkesLiuWu2014,KarmakarWu2020}.  The ordinary-SGD theorem that supplies our one-step contraction module is \citet{LiLouRichterWu2026}.  The probability results used below are attributed to those sources.  Our contribution is their verification through explicit augmented-state contraction, together with the comparison of moment-specific, general quadratic, and exact scalar stability conditions.

\begin{table}[t]
\centering
\caption{Representative comparison with the closest literature.  ``Constant'' refers to a fixed learning rate and momentum parameter.}
\label{tab:literature}
\renewcommand{\arraystretch}{1.13}
\scriptsize
\begin{tabularx}{\textwidth}{>{\raggedright\arraybackslash}p{0.19\textwidth}>{\raggedright\arraybackslash}p{0.17\textwidth}>{\raggedright\arraybackslash}p{0.27\textwidth}>{\raggedright\arraybackslash}X}
\toprule
Reference & Algorithm/regime & Main object & Difference from the present paper\\
\midrule
\citet{ChenWu2016} & Nonlinear AR($\infty$); finite-lag special cases & Stochastic-Lipschitz stability, functional dependence, CLT and scalar strong approximation & Direct predecessor; the AR(2) specialization gives $q(1+2\beta)<1$. Signed Nesterov metrics yield additional regimes.\\
\citet{CanGurbuzbalabanZhu2019} & HB, NAG, APG; constant & Wasserstein convergence to an invariant law & Structured oracle and finite-variance control; no direct nonlinear synchronous $L^p$ GMC or physical-dependence limit theory.\\
\citet{AssranRabbat2020} & NAG; constant & Optimization error and steady-state neighborhood & Smooth strongly convex objectives with bounded variance; does not identify a causal GMC process or arbitrary-initialization CLT.\\
\citet{GitmanLangZhangXiao2019} & Unified momentum; local quadratic & Stability and stationary covariance & Local linearization rather than global nonlinear contraction.\\
\citet{GadatPanloupSaadane2018,BarakatEtAl2021} & Heavy ball/general momentum; decreasing gains & Convergence, fluctuations, trap avoidance & Vanishing-step asymptotics rather than a fixed-parameter invariant process.\\
\citet{LiXiaoYang2023,TangLiuZhangChen2023} & SGD, momentum, NAG & CLTs and averaged-iterate asymptotics & Primarily decreasing gains/local asymptotics; no GMC construction for the constant NAG state.\\
\citet{GaneshEtAl2023} & SHB and stochastic NAG & Sample-complexity limitations & Compares optimization performance; does not develop stationary nonlinear time-series theory.\\
\citet{YamadaSatoIiduka2026,DangEtAl2025} & Momentum under heavy tails & Finite-time convergence or generalization stability & Different targets and assumptions; our result gives pathwise-in-data synchronous $L^p$ forgetting and stationary consequences.\\
\citet{WangYurtsever2026,LeiWangYuan2026} & Generalized SGDM incl. NAG & Optimization and algorithmic stability & Broad algorithmic coverage, but not invariant-law GMC or strong approximation.\\
\citet{JagannathJonesMcCormickSarangian2026} & Momentum; high dimension & High-dimensional scaling limits & Proportional asymptotics rather than fixed-dimensional global GMC.\\
This paper & NAG; constant & Moment-specific GMC, quadratic certificates, physical dependence and stationary inference & Direct criterion: small-step $\beta<\mu/(\mu+L_p)$. Power criterion: every fixed $\beta<1$, generally much smaller step. Limit theory uses the corresponding moment order and every deterministic initial pair.\\
\bottomrule
\end{tabularx}
\end{table}

The substantive claims are the explicit finite-moment coupling criteria, their carefully qualified stability regions, and the verification of a stationary inference framework for Nesterov's augmented state.  The power argument and the exact random-curvature comparison strengthen the distinction between existence of a contractive regime and a useful quantitative certificate.  We do not claim that the averaged-sequence coordinate, the Perron theorem, the polytopic LMI argument, causal-process limit theorems, or quadratic AR(2) calculations are individually new. In particular, the present work is a Nesterov-specific development of nonlinear autoregressive stability and inference in the sense of \citet{ChenWu2016}, not an extension of all their infinite-lag, spatial, or long-memory results.

\FloatBarrier
\section{Setup and the ordinary-SGD contraction module}
\label{sec:setup}

Following the stochastic-approximation framework initiated by \citet{RobbinsMonro1951} and the constant-step viewpoint of \citet{Pflug1986}, let $(X_k)_{k\in\mathbb Z}$ be i.i.d. random elements on a measurable space $(\cX,\mathcal A)$ with law $\Pi$.  Put
\[
 G(\theta,x)=\nabla_\theta g(\theta,x),
 \qquad
 m(\theta)=\E G(\theta,X_0).
\]
For a random vector $Z$ and $p>0$, write $\norm{Z}_p=(\E\abs Z^p)^{1/p}$.  All vector norms are Euclidean and matrix inequalities are in the Loewner order.

\begin{assumption}[Mean strong monotonicity]
\label{ass:strong}
There is $\mu>0$ such that
\begin{equation}
 \ip{m(\theta)-m(\vartheta)}{\theta-\vartheta}
 \geq \mu\abs{\theta-\vartheta}^2,
 \qquad \theta,\vartheta\in\R^d.
 \label{eq:strong}
\end{equation}
The equation $m(\theta)=0$ has a solution, necessarily unique by \eqref{eq:strong}, denoted $\theta^\star$.
\end{assumption}

\begin{assumption}[Stochastic $L^p$ Lipschitz continuity]
\label{ass:lp}
For a fixed $p>1$, there is $L_p<\infty$ such that
\begin{equation}
 \norm{G(\theta,X_0)-G(\vartheta,X_0)}_p
 \leq L_p\abs{\theta-\vartheta},
 \qquad \theta,\vartheta\in\R^d,
 \label{eq:lp-lip}
\end{equation}
and
\begin{equation}
 \sigma_p:=\norm{G(\theta^\star,X_0)}_p<\infty.
 \label{eq:sigma-p}
\end{equation}
\end{assumption}

Assumption~\ref{ass:strong} concerns only the mean field.  A realization $\theta\mapsto g(\theta,x)$ need not be convex.  Assumption~\ref{ass:lp} allows fractional moments and does not require a finite variance when $p<2$.  Necessarily $\mu\leq L_p$: for $v=\theta-\vartheta\ne0$,
\[
 \mu|v|^2\leq\langle v,\E\{G(\theta,X)-G(\vartheta,X)\}\rangle
 \leq |v|\|G(\theta,X)-G(\vartheta,X)\|_p\leq L_p|v|^2.
\]
The GMC arguments through Section~\ref{sec:power-extension}, the mean-only certificate in Section~\ref{sec:mean-lmi}, and their dependence and limit consequences use monotonicity and Lipschitz continuity, not that $G$ is a gradient.  They therefore also apply to the same unconstrained extrapolated stochastic monotone-operator recursion.  This observation does not assert a theorem for projected variational inequalities or proximal acceleration; the samplewise Hessian-sector section specifically uses symmetric gradient structure.

Define the ordinary-SGD random map
\[
 F_x(\theta)=\theta-\gamma G(\theta,x).
\]
The contraction coefficient inherited from the ordinary-SGD theory is
\begin{equation}
q_{\gamma,p}^p=
\begin{cases}
 1-p\mu\gamma+2^{2-p}L_p^p\gamma^p,&1<p<2,\\[1.5mm]
 (1+\gamma L_p)^p-p\gamma L_p-p\mu\gamma,&p\geq2.
\end{cases}
\label{eq:qdef}
\end{equation}
For $1<p<2$, $q_{\gamma,p}<1$ whenever
\begin{equation}
0<\gamma<\gamma_p
:=\left(\frac{p\mu}{2^{2-p}L_p^p}\right)^{1/(p-1)}.
\label{eq:gamma-smallp}
\end{equation}
For $p\geq2$, $\gamma_p$ is the unique positive root of
\begin{equation}
 (1+\gamma L_p)^p=1+p\gamma L_p+p\mu\gamma,
 \label{eq:gamma-largep}
\end{equation}
and $q_{\gamma,p}<1$ for $0<\gamma<\gamma_p$.

\begin{proposition}[Ordinary-SGD $L^p$ contraction]
\label{prop:base}
Under Assumptions~\ref{ass:strong}--\ref{ass:lp}, for every $\theta,\vartheta\in\R^d$,
\begin{equation}
 \norm{F_{X_0}(\theta)-F_{X_0}(\vartheta)}_p
 \leq q_{\gamma,p}\abs{\theta-\vartheta}.
 \label{eq:base-contract}
\end{equation}
The same inequality holds conditionally when $\theta$ and $\vartheta$ are random and measurable with respect to a sigma-field independent of $X_0$.
\end{proposition}

The deterministic-state statement is Theorem~1 of the supplied revised manuscript of \citet{LiLouRichterWu2026}; its public author listing is given in the bibliography.  A self-contained proof, including the nonquadratic power inequalities for $1<p<2$ and $p\geq2$, is given in \Cref{app:base}.  Conditional validity follows by applying the deterministic inequality at the realized pair and then integrating.

\section{Nesterov acceleration as an iterated random function}
\label{sec:state}

Let
\begin{equation}
 S_k=\Theta_k-\Theta_{k-1}
 \label{eq:increment}
\end{equation}
be the increment.  Then \eqref{eq:nag} becomes
\begin{equation}
\begin{split}
 Y_k&=\Theta_k+\beta S_k,\\
 S_{k+1}&=\beta S_k-\gamma G(Y_k,X_{k+1}),\\
 \Theta_{k+1}&=\Theta_k+S_{k+1}.
\end{split}
\label{eq:nag-increment}
\end{equation}
The state $(Y_k,S_k)$ is Markov because $\Theta_k=Y_k-\beta S_k$.  Its random update is
\begin{equation}
 \Phi_x(y,s)=
 \begin{pmatrix}
 y+\beta^2s-(1+\beta)\gamma G(y,x)\\
 \beta s-\gamma G(y,x)
 \end{pmatrix}.
 \label{eq:phi}
\end{equation}
The original iterates are recovered by
\begin{equation}
 \Theta_k=Y_k-\beta S_k,
 \qquad
 \Theta_{k-1}=Y_k-(1+\beta)S_k.
 \label{eq:inverse-state}
\end{equation}
Thus $(Y,S)$ and $(\Theta_k,\Theta_{k-1})$ are equivalent linear coordinates.

For two trajectories driven by the same innovations, write
\begin{equation}
 x_k=\Theta_k-\widetilde\Theta_k,
 \quad
 s_k=S_k-\widetilde S_k,
 \quad
 y_k=Y_k-\widetilde Y_k=x_k+\beta s_k,
 \label{eq:diffs}
\end{equation}
and
\begin{equation}
 h_{k+1}=G(Y_k,X_{k+1})-G(\widetilde Y_k,X_{k+1}).
 \label{eq:h}
\end{equation}
Then
\begin{equation}
 x_{k+1}=y_k-\gamma h_{k+1},
 \qquad
 s_{k+1}=\beta s_k-\gamma h_{k+1}.
 \label{eq:diff-recursion}
\end{equation}
These identities are the algebraic core of the paper.

\begin{definition}[Product $L^p$ GMC]
\label{def:product-gmc}
Let $Z_k^z=(Y_k^z,S_k^z)$ denote \eqref{eq:phi} started at $z=(y,s)$.  For positive weights $a,b$, define
\[
 \mathfrak D_{p;a,b}(Z,\widetilde Z)
 =a\norm{Y-\widetilde Y}_p+b\norm{S-\widetilde S}_p.
\]
We say that \eqref{eq:phi} is product-$L^p$ GMC with rate $r<1$ if
\[
 \mathfrak D_{p;a,b}(Z_k^z,Z_k^{\widetilde z})
 \leq r^k\{a\abs{y-\widetilde y}+b\abs{s-\widetilde s}\}
\]
for every deterministic $z,\widetilde z$ under synchronous coupling.
\end{definition}

Because the coordinate transform \eqref{eq:inverse-state} is nonsingular, product GMC is equivalent, up to fixed constants depending on $\beta$ and the weights, to GMC of the original two-iterate state.

\begin{definition}[Classical replacement-of-the-past GMC]
\label{def:classical-gmc}
Let $Z_k^\circ=H(\ldots,X_{k-1},X_k)$ be a causal stationary state.  For $m\geq1$, let
\[
 Z_{k,\{m\}}^\circ
 =H(\ldots,X'_{k-m-1},X'_{k-m},X_{k-m+1},\ldots,X_k),
\]
where $(X_j')$ is an independent copy of $(X_j)$.  For a state metric $d$, the process is GMC of order $p$ if there are $C<\infty$ and $r\in(0,1)$ such that
\begin{equation}
 \norm{d(Z_k^\circ,Z_{k,\{m\}}^\circ)}_p\leq Cr^m,
 \qquad m\geq1.
 \label{eq:classical-gmc}
\end{equation}
This is the replacement-of-the-infinite-past formulation used in nonlinear time-series theory \citep{WuShao2004}.  The synchronous contraction established below implies \eqref{eq:classical-gmc}; see Corollary~\ref{cor:classical-gmc}.
\end{definition}

\subsection{Connection with Chen and Wu's autoregressive framework}
\label{sec:ar-connection}

The original iterates themselves have the nonlinear AR(2) representation
\begin{equation}
 \Theta_{k+1}=\mathcal T_{X_{k+1}}(\Theta_k,\Theta_{k-1}),
 \qquad
 \mathcal T_x(u,v)=F_x\{(1+\beta)u-\beta v\},
 \label{eq:ar2-map}
\end{equation}
where $F_x(w)=w-\gamma G(w,x)$.  This gives a concrete connection with Condition~1 and Theorem~1 of \citet{ChenWu2016}, stated there for scalar nonlinear autoregressions.  We supply the finite-dimensional two-lag argument below rather than treating a vector theorem as a quoted result.

Let $q\geq0$ be any valid one-step bound
$\|F_X(w)-F_X(w')\|_p\leq q|w-w'|$, and assume $\sigma_p<\infty$.
One may take $q=q_{\gamma,p}$ from Proposition~\ref{prop:base}, or use a sharper model-specific coefficient.  Since
$(1+\beta)\theta^\star-\beta\theta^\star=\theta^\star$,
\[
 \|\mathcal T_X(\theta^\star,\theta^\star)-\theta^\star\|_p
 =\gamma\sigma_p<\infty.
\]
Thus shifting by $\theta^\star$ verifies the anchor moment.  For deterministic input pairs,
\begin{align}
 \|\mathcal T_X(u,v)-\mathcal T_X(u',v')\|_p
 &\leq q|(1+\beta)(u-u')-\beta(v-v')|\notag\\
 &\leq a_1|u-u'|+a_2|v-v'|,
 \qquad a_1=(1+\beta)q,\quad a_2=\beta q.
 \label{eq:ar2-lipschitz}
\end{align}
The nonnegative lag-coefficient condition is therefore
\begin{equation}
 a_1+a_2=(1+2\beta)q<1.
 \label{eq:ar2-condition}
\end{equation}

\begin{proposition}[Two-lag specialization of autoregressive contraction]
\label{prop:ar2}
Let $p>1$, $0<\beta<1$, $q>0$, and suppose the one-step and anchor bounds above hold.  Under \eqref{eq:ar2-condition}, set
\begin{equation}
 r_{\rm AR}=
 \frac{(1+\beta)q+\sqrt{(1+\beta)^2q^2+4\beta q}}{2},
 \qquad
 w_{\rm AR}=\frac{\beta q}{r_{\rm AR}}>0.
 \label{eq:ar2-rate}
\end{equation}
Then $r_{\rm AR}<1$, and every synchronous pair with deterministic initial differences $x_0,x_{-1}$ satisfies
\begin{equation}
 \|x_k\|_p+w_{\rm AR}\|x_{k-1}\|_p
 \leq r_{\rm AR}^{\,k}
       (|x_0|+w_{\rm AR}|x_{-1}|),\qquad k\geq0.
 \label{eq:ar2-contraction}
\end{equation}
There is a unique causal stationary solution with a finite $p$th moment.  The same bound holds for an initial state independent of future innovations, with $L^p$ norms on the right.  In particular, geometric coupling to stationarity and exponential physical dependence hold for the Nesterov state.
\end{proposition}

\begin{proof}
Independence of a fresh innovation permits conditioning on the two past states in \eqref{eq:ar2-lipschitz}.  Conditional expectation followed by Minkowski's inequality gives
\[
 \begin{pmatrix}\|x_{k+1}\|_p\\\|x_k\|_p\end{pmatrix}
 \leq C_{\rm AR}
 \begin{pmatrix}\|x_k\|_p\\\|x_{k-1}\|_p\end{pmatrix},
 \qquad
 C_{\rm AR}=\begin{pmatrix}a_1&a_2\\1&0\end{pmatrix}.
\]
Its characteristic polynomial is $\lambda^2-a_1\lambda-a_2$.  The positive root is $r_{\rm AR}$, and the other root is negative with smaller absolute value.  Consequently
$r_{\rm AR}<1$ if and only if $1-a_1-a_2>0$.
Moreover $(1,w_{\rm AR})C_{\rm AR}=r_{\rm AR}(1,w_{\rm AR})$.
Multiplication by this positive row vector and iteration give \eqref{eq:ar2-contraction}.

For completeness, write
$\Psi_x(u,v)=(\mathcal T_x(u,v),u)$ and equip random two-iterate states with
$\mathfrak D(U,V)=\|U_1-V_1\|_p+w_{\rm AR}\|U_2-V_2\|_p$.
The one-step anchor distance at $z^\star=(\theta^\star,\theta^\star)$ is $\gamma\sigma_p$.
For backward compositions $B_n=\Psi_{X_0}\circ\cdots\circ\Psi_{X_{-n+1}}(z^\star)$,
\[
 \mathfrak D(B_{n+1},B_n)\leq r_{\rm AR}^{\,n}\gamma\sigma_p.
\]
Summability and completeness yield a finite-$p$-moment, past-measurable limit.
The same conditional contraction permits passage to the limit through a fresh random map, yielding the stationary recursion.  Coupling two causal stationary solutions and iterating \eqref{eq:ar2-contraction} proves uniqueness in this moment class.  Forward coupling gives convergence from every deterministic initial pair.  Replacing the innovation at time zero creates a finite-$p$ difference at that time; the common future then contracts it geometrically.  Coordinate equivalence in \eqref{eq:inverse-state} gives the assertions for $(Y_k,S_k)$.
\end{proof}

When $\beta=0$, the position recursion is ordinary SGD.  When $q=0$, its next position is independent of the input state under synchronous coupling, so both coordinates of the two-iterate difference vanish after two updates.  These degenerate cases need no positive Perron weight.

For the stationary process, let $\delta_{k,p}^{\Theta}$ denote its $L^p$ change when only $X_0$ is replaced by an independent copy, as formally defined in Section~\ref{sec:limits}.  The dependence argument gives, with $\delta_{j,p}^{\Theta}=0$ for $j<0$,
\begin{equation}
 \delta_{k,p}^{\Theta}
 \leq a_1\delta_{k-1,p}^{\Theta}+a_2\delta_{k-2,p}^{\Theta},
 \qquad k\geq1.
 \label{eq:ar2-dependence}
\end{equation}
This is the finite-lag instance of equation~(36) in the proof of Corollary~6 of \citet{ChenWu2016}.  Here the $2\times2$ comparison proves the exponential rate directly.

\paragraph{Why a separate Nesterov metric remains useful.}
For every fixed $\beta>0$, $q_{\gamma,p}\to1$ as $\gamma\downarrow0$; hence \eqref{eq:ar2-condition} fails for all sufficiently small steps.  It discards the negative sign multiplying the previous iterate.  By contrast, Section~\ref{sec:direct} uses position--increment coordinates, and Section~\ref{sec:power-extension} proves contraction at every fixed $\beta<1$ for sufficiently small steps by preserving signed information.  These are additional certificates, not a claim that the lag-coefficient approach is uniformly dominated.

Indeed, for $G(\theta,X)=h(\theta-\theta^\star)-\xi$, $h>0$, the exact one-step coefficient is $q=|1-\gamma h|$ at every finite innovation-moment order.  Condition~\eqref{eq:ar2-condition} then becomes
\begin{equation}
 \frac{2\beta}{1+2\beta}<\gamma h<
 \frac{2(1+\beta)}{1+2\beta}.
 \label{eq:ar2-quadratic-band}
\end{equation}
The upper endpoint agrees with Theorem~\ref{thm:quadratic-stability}, but the positive lower endpoint is artificial: the exact region has lower endpoint zero.  At $\beta=0.9$, for example, this autoregressive certificate covers
$9/14<\gamma h<19/14$, including steps inaccessible to the scalar specialization of Theorem~\ref{thm:direct}.  Conversely, the all-momentum theorems cover sufficiently small steps excluded by \eqref{eq:ar2-quadratic-band}.  Failure of any one sufficient condition does not imply instability.

\section{Direct \texorpdfstring{$L^p$}{Lp} geometric moment contraction}
\label{sec:direct}

Put
\begin{equation}
 q=q_{\gamma,p},
 \qquad
 \ell=\gamma L_p.
 \label{eq:qell}
\end{equation}
Conditionally on the past, Proposition~\ref{prop:base}, Assumption~\ref{ass:lp}, and \eqref{eq:diff-recursion} give
\begin{align}
 \norm{s_{k+1}}_p
 &\leq \ell\norm{y_k}_p+\beta\norm{s_k}_p,
 \label{eq:s-bound}\\
 \norm{y_{k+1}}_p
 &=\norm{x_{k+1}+\beta s_{k+1}}_p\notag\\
 &\leq \norm{x_{k+1}}_p+\beta\norm{s_{k+1}}_p\notag\\
 &\leq (q+\beta\ell)\norm{y_k}_p+\beta^2\norm{s_k}_p.
 \label{eq:y-bound}
\end{align}
Therefore
\begin{equation}
 \begin{pmatrix}
 \norm{y_{k+1}}_p\\[1mm]
 \norm{s_{k+1}}_p
 \end{pmatrix}
 \leq
 M^{\NAG}_{\gamma,\beta,p}
 \begin{pmatrix}
 \norm{y_k}_p\\[1mm]
 \norm{s_k}_p
 \end{pmatrix},
 \qquad
 M^{\NAG}_{\gamma,\beta,p}
 =\begin{pmatrix}
 q+\beta\ell&\beta^2\\
 \ell&\beta
 \end{pmatrix},
 \label{eq:comparison}
\end{equation}
where inequalities between vectors are componentwise.

\begin{theorem}[Direct Perron $L^p$ contraction]
\label{thm:direct}
Suppose Assumptions~\ref{ass:strong}--\ref{ass:lp} hold for some $p>1$, $q_{\gamma,p}<1$, and $0<\beta<1$.  Define
\begin{equation}
 r_{\gamma,\beta,p}
 =\frac{q+\beta(1+\ell)+
 \sqrt{(q-\beta+\beta\ell)^2+4\beta^2\ell}}{2}.
 \label{eq:r-direct}
\end{equation}
Then $r_{\gamma,\beta,p}$ is the Perron root of \eqref{eq:comparison}, and
\begin{equation}
 r_{\gamma,\beta,p}<1
 \quad\Longleftrightarrow\quad
 \beta\gamma L_p<(1-\beta)(1-q_{\gamma,p}).
 \label{eq:direct-condition}
\end{equation}
When \eqref{eq:direct-condition} holds, let
\begin{equation}
 \eta_{\gamma,\beta,p}
 =\frac{\beta^2}{r_{\gamma,\beta,p}-\beta}>0.
 \label{eq:eta}
\end{equation}
For every synchronously coupled pair of trajectories,
\begin{equation}
 \norm{y_k}_p+\eta_{\gamma,\beta,p}\norm{s_k}_p
 \leq r_{\gamma,\beta,p}^k
 \left(\abs{y_0}+\eta_{\gamma,\beta,p}\abs{s_0}\right).
 \label{eq:direct-contraction}
\end{equation}
Consequently the Nesterov state is product-$L^p$ GMC.  It has a unique causal stationary solution $Z_k^\circ=(Y_k^\circ,S_k^\circ)$ with finite $p$th moment, and for any deterministic state $z$,
\begin{equation}
 \norm{Y_k^z-Y_k^\circ}_p
 +\eta_{\gamma,\beta,p}\norm{S_k^z-S_k^\circ}_p
 \leq r_{\gamma,\beta,p}^k
 \left\{\norm{y-Y_0^\circ}_p
 +\eta_{\gamma,\beta,p}\norm{s-S_0^\circ}_p\right\}.
 \label{eq:direct-stationary-coupling}
\end{equation}
For $\beta=0$, use $\eta=0$ for the position component and Proposition~\ref{prop:base}; the redundant increment is then recovered from successive positions.
\end{theorem}

\begin{proof}
The characteristic polynomial of $M^{\NAG}_{\gamma,\beta,p}$ is
\[
 \lambda^2-\{q+\beta(1+\ell)\}\lambda+\beta q.
\]
Its larger root is \eqref{eq:r-direct}.  Because $q<1$ and $\beta<1$, the Perron root is below one exactly when
\begin{align*}
 \det(I_2-M^{\NAG}_{\gamma,\beta,p})
 &=(1-q-\beta\ell)(1-\beta)-\beta^2\ell\\
 &=(1-\beta)(1-q)-\beta\ell>0,
\end{align*}
which is \eqref{eq:direct-condition}.  The same inequality implies $q+\beta\ell<1$, so no additional diagonal condition is needed.

Let $r=r_{\gamma,\beta,p}$.  A positive left Perron vector can be normalized as $(1,\eta)$ because
\[
 \beta^2+\eta\beta=r\eta
 \quad\Longleftrightarrow\quad
 \eta=\frac{\beta^2}{r-\beta}.
\]
Multiplying \eqref{eq:comparison} by $(1,\eta)$ and iterating proves \eqref{eq:direct-contraction}.

For stationarity, consider the backward compositions
\[
 B_n(z)=\Phi_{X_0}\circ\Phi_{X_{-1}}\circ\cdots\circ
 \Phi_{X_{-n+1}}(z).
\]
At the anchor $z^\star=(\theta^\star,0)$,
\[
 \Phi_X(z^\star)-z^\star
 =\bigl(-(1+\beta)\gamma G(\theta^\star,X),
       -\gamma G(\theta^\star,X)\bigr),
\]
which has finite product $L^p$ norm by \eqref{eq:sigma-p}.  By applying \eqref{eq:direct-contraction} to the first $n$ common maps,
\[
 \mathfrak D_p\{B_{n+1}(z^\star),B_n(z^\star)\}
 \leq r^n\gamma\{1+\beta+\eta\}\sigma_p.
\]
The geometric series is summable, so $(B_n(z^\star))$ is Cauchy in $L^p(\R^{2d})$.  Its limit $Z_0^\circ$ is measurable with respect to $(\ldots,X_{-1},X_0)$, has finite $p$th moment, and satisfies the stationary recursion by the conditional $L^p$ continuity bound and shift covariance.  The contraction shows that a backward limit started from any deterministic anchor is the same.  If two stationary solutions are coupled with the same innovations, iteration of \eqref{eq:direct-contraction} and stationarity force their distance to be zero; uniqueness follows.  Forward coupling yields \eqref{eq:direct-stationary-coupling}.  This is the standard backward-iteration construction for contractive iterated random functions, written here in the product $L^p$ metric.
\end{proof}

\begin{corollary}[Classical GMC and Wasserstein mixing]
\label{cor:classical-gmc}
Under Theorem~\ref{thm:direct}, equip the state space with
\[
 d_{\eta}\{(y,s),(\widetilde y,\widetilde s)\}
 =\abs{y-\widetilde y}+\eta_{\gamma,\beta,p}\abs{s-\widetilde s}.
\]
Then the stationary state satisfies the classical GMC bound
\begin{equation}
 \norm{d_\eta(Z_k^\circ,Z_{k,\{m\}}^\circ)}_p
 \leq C r_{\gamma,\beta,p}^{m},
 \qquad m\geq1.
 \label{eq:tail-gmc-direct}
\end{equation}
If $P^k(z,\cdot)$ denotes the $k$-step transition law and $\pi_{\gamma,\beta}$ the unique invariant law, then the $p$-Wasserstein distance associated with $d_\eta$ obeys
\begin{equation}
 W_{p,d_\eta}\{P^k(z,\cdot),\pi_{\gamma,\beta}\}
 \leq r_{\gamma,\beta,p}^{k}
 \left(\int d_\eta(z,z')^p\,\pi_{\gamma,\beta}(dz')\right)^{1/p}.
 \label{eq:wasserstein-direct}
\end{equation}
An analogous statement holds under Theorem~\ref{thm:allbeta} with $p=2$ and the metric $d_P(z,z')=\{(z-z')^\trans(P_\beta\otimes I_d)(z-z')\}^{1/2}$.
\end{corollary}

\begin{proof}
At time $k-m$, the two replacement-of-the-past versions have finite $p$th distance and use the same last $m$ innovations.  Applying \eqref{eq:direct-contraction} through these common maps gives \eqref{eq:tail-gmc-direct}.  For \eqref{eq:wasserstein-direct}, couple a trajectory started from $z$ with a stationary trajectory and use the same future innovations.  The synchronous bound controls the $L^p$ transportation cost, and taking the infimum over couplings proves the claim.  The quadratic-metric case is identical.
\end{proof}

\begin{remark}[Why the direct condition restricts large momentum]
\label{rem:direct-limit}
As $\gamma\downarrow0$, $q_{\gamma,p}=1-\mu\gamma+o(\gamma)$, so \eqref{eq:direct-condition} approaches
\begin{equation}
 \beta<\frac{\mu}{\mu+L_p}.
 \label{eq:direct-smallgamma}
\end{equation}
The restriction is caused by taking absolute values separately in \eqref{eq:s-bound}--\eqref{eq:y-bound}.  It is robust to heavy tails but discards the stabilizing sign of the cross term.  The quadratic metric in \Cref{sec:allbeta} removes the qualitative restriction in $L^2$, not the quantitative conservatism.  Theorem~\ref{thm:allp} treats other moment orders with a different, more restrictive step condition.
\end{remark}

\begin{remark}[Scalar quadratic specialization]
For $G(\theta)=h(\theta-\theta^\star)$, $q_{\gamma,2}=\abs{1-\gamma h}$ and $\ell=\gamma h$.  The direct criterion becomes
\[
 \beta\gamma h<(1-\beta)\{1-\abs{1-\gamma h}\}.
\]
It allows $0<\gamma h<2(1-\beta)$ only when $\beta<1/2$.  By contrast, the exact Nesterov stability region in \Cref{sec:quadratic} remains nonempty for every $\beta<1$.  This gap cleanly separates a heavy-tail-robust sufficient condition from a sharp signed linear analysis.
\end{remark}

\section{Mean-square contraction for every momentum parameter}
\label{sec:allbeta}

The direct theorem treats every $p>1$, but its componentwise comparison is conservative when $\beta$ is large.  We now exploit the signed algebra of \eqref{eq:diff-recursion}.  Throughout this section, Assumptions~\ref{ass:strong}--\ref{ass:lp} hold with $p=2$; write $L=L_2$ and
\begin{equation}
 d_\beta=1-\beta.
 \label{eq:dbeta}
\end{equation}
For $0<\beta<1$, introduce the standard averaged-sequence (or slow) coordinate.  This transformation is used in momentum analyses by \citet{GhadimiFeyzmahdavianJohansson2015,YangLinLi2016,LiuGaoYin2020}; see in particular Proposition~5 and equation~(27) of \citet{SebbouhGowerDefazio2021}, with the constant extrapolation coefficient $\beta/(1-\beta)$.  We use its \emph{synchronous difference} as a contraction coordinate:
\begin{equation}
 u_k=x_k+\frac{\beta}{d_\beta}s_k.
 \label{eq:u}
\end{equation}
Since $y_k=x_k+\beta s_k$, the two useful identities are
\begin{equation}
 u_k=y_k+\frac{\beta^2}{d_\beta}s_k,
 \qquad
 u_{k+1}=u_k-\frac{\gamma}{d_\beta}h_{k+1}.
 \label{eq:u-identities}
\end{equation}
The second identity is exact.  The transformation itself is not new; we use it to retain the signed coupling terms that a componentwise comparison loses.

Set
\begin{equation}
 a_\beta=\frac{\beta^2}{d_\beta^2},
 \qquad
 c_\beta=\frac{\beta^2}{d_\beta},
 \label{eq:acbeta}
\end{equation}
and define
\begin{equation}
 V_k=\abs{u_k}^2+a_\beta\abs{s_k}^2
 =\abs{y_k+c_\beta s_k}^2+a_\beta\abs{s_k}^2.
 \label{eq:V}
\end{equation}
In $(y,s)$ coordinates,
\begin{equation}
 V_k=
 \begin{pmatrix}y_k\\s_k\end{pmatrix}^{\!\trans}
 (P_\beta\otimes I_d)
 \begin{pmatrix}y_k\\s_k\end{pmatrix},
 \qquad
 P_\beta=
 \begin{pmatrix}
 1&c_\beta\\
 c_\beta&c_\beta^2+a_\beta
 \end{pmatrix}\succ0.
 \label{eq:Pbeta}
\end{equation}

\begin{theorem}[All-$\beta$ geometric mean-square contraction]
\label{thm:allbeta}
Suppose Assumptions~\ref{ass:strong}--\ref{ass:lp} hold with $p=2$.  Let $0<\beta<1$ and assume
\begin{equation}
 0<\gamma<\Gamma_{\mathrm{NAG}}(\beta)
 :=\frac{2\mu(1-\beta)^2}
 {L^2\{1-\beta+2\beta^2\}}.
 \label{eq:allbeta-bound}
\end{equation}
Define
\begin{align}
 A_{\gamma,\beta}
 &=\frac{2\mu\gamma}{d_\beta}
   -\frac{1+\beta^2}{d_\beta^2}L^2\gamma^2,
 \label{eq:A}\\
 B_{\gamma,\beta}
 &=\frac{\beta^2(1+\beta)}{d_\beta^2}L\gamma,
 \label{eq:B}\\
 D_\beta
 &=\frac{\beta^2(1+\beta)}{d_\beta},
 \label{eq:D}
\end{align}
and
\begin{equation}
 K_{\gamma,\beta}=
 \begin{pmatrix}
 A_{\gamma,\beta}&-B_{\gamma,\beta}\\
 -B_{\gamma,\beta}&D_\beta
 \end{pmatrix}.
 \label{eq:K}
\end{equation}
Then $K_{\gamma,\beta}\succ0$.  With
\begin{equation}
 \kappa_{\gamma,\beta}
 =\min\left\{\frac12,
 \frac{\lambda_{\min}(K_{\gamma,\beta})}
 {\lambda_{\max}(P_\beta)}\right\}>0,
 \qquad
 r_{\mathrm{MS}}=\sqrt{1-\kappa_{\gamma,\beta}}<1,
 \label{eq:kappa}
\end{equation}
we have, conditionally on the past,
\begin{equation}
 \E(V_{k+1}\mid\cF_k)
 \leq(1-\kappa_{\gamma,\beta})V_k.
 \label{eq:V-contract}
\end{equation}
Consequently
\begin{equation}
 \norm{V_k^{1/2}}_2
 \leq r_{\mathrm{MS}}^k V_0^{1/2},
 \label{eq:allbeta-contraction}
\end{equation}
the two-iterate Nesterov state is geometrically mean-square contracting, and it admits a unique causal stationary law with a finite second moment.  When $\beta=0$, Proposition~\ref{prop:base} gives the ordinary-SGD condition $0<\gamma<2\mu/L^2$.
\end{theorem}

\begin{proof}
Let
\[
 \overline h_k
 =\E(h_{k+1}\mid\cF_k)
 =m(Y_k)-m(\widetilde Y_k).
\]
Using \eqref{eq:u-identities}, $s_{k+1}=\beta s_k-\gamma h_{k+1}$, and $a=a_\beta$ gives
\begin{align}
 \E(V_{k+1}\mid\cF_k)
 &=V_k-\frac{2\gamma}{d_\beta}\ip{u_k}{\overline h_k}
 -2a\beta\gamma\ip{s_k}{\overline h_k}\notag\\
 &\quad+\gamma^2(d_\beta^{-2}+a)
 \E(\abs{h_{k+1}}^2\mid\cF_k)
 -a(1-\beta^2)\abs{s_k}^2.
 \label{eq:V-expansion}
\end{align}
Because $u_k=y_k+c_\beta s_k$, mean strong monotonicity and stochastic $L^2$ Lipschitz continuity imply
\begin{align}
 \ip{y_k}{\overline h_k}&\geq\mu\abs{y_k}^2,
 \label{eq:mono-y}\\
 \abs{\overline h_k}&\leq L\abs{y_k},
 \label{eq:mean-lip}\\
 \E(\abs{h_{k+1}}^2\mid\cF_k)&\leq L^2\abs{y_k}^2.
 \label{eq:h2}
\end{align}
Substituting $c_\beta=\beta^2/d_\beta$ and $a=\beta^2/d_\beta^2$ into \eqref{eq:V-expansion} yields
\begin{align}
 \E(V_{k+1}\mid\cF_k)
 &\leq V_k-A_{\gamma,\beta}\abs{y_k}^2
 +2B_{\gamma,\beta}\abs{y_k}\abs{s_k}
 -D_\beta\abs{s_k}^2.
 \label{eq:V-K}
\end{align}
The determinant of $K_{\gamma,\beta}$ simplifies to
\begin{equation}
 \det K_{\gamma,\beta}
 =\frac{\beta^2(1+\beta)\gamma}{d_\beta^4}
 \left[2\mu d_\beta^2
 -L^2\gamma\{d_\beta+2\beta^2\}\right].
 \label{eq:detK}
\end{equation}
Thus \eqref{eq:allbeta-bound} gives $\det K_{\gamma,\beta}>0$.  It also implies $A_{\gamma,\beta}>0$ because
\[
 \frac{d_\beta^2}{d_\beta+2\beta^2}
 \leq\frac{d_\beta}{1+\beta^2}.
\]
Hence $K_{\gamma,\beta}\succ0$.

For $z=(y,s)\in\R^{2d}$, the last three terms of \eqref{eq:V-K} are bounded above by
\[
 -\lambda_{\min}(K_{\gamma,\beta})
 (\abs y^2+\abs s^2).
\]
On the other hand,
\[
 V=z^\trans(P_\beta\otimes I_d)z
 \leq\lambda_{\max}(P_\beta)(\abs y^2+\abs s^2).
\]
Equations \eqref{eq:kappa} and \eqref{eq:V-K} therefore imply \eqref{eq:V-contract}.  Iteration and integration give \eqref{eq:allbeta-contraction}.  The backward-iteration construction used in Theorem~\ref{thm:direct}, now in the norm generated by $P_\beta\otimes I_d$, proves existence, uniqueness, and finite second moment of the stationary law.
\end{proof}

\begin{remark}[Comparison with the analogous heavy-ball certificate]
\label{rem:hb-comparison}
For heavy ball, the gradient difference is evaluated at $x$, and the same standard coordinate $u=x+\beta s/d_\beta$ satisfies $u_+=u-\gamma h/d_\beta$.  For $|u|^2+a|s|^2$ the corresponding coefficients are
\[
 A(a)=\frac{2\mu\gamma}{d_\beta}-(d_\beta^{-2}+a)L^2\gamma^2,\quad
 B(a)=\gamma L(\beta/d_\beta^2+a\beta),\quad
 D(a)=a(1-\beta^2).
\]
Writing $t=ad_\beta^2$, the determinant condition becomes
\[
 \gamma<\frac{2\mu(1+\beta)d_\beta^2t}
 {L^2\{t^2+(1+\beta^2)t+\beta^2\}}.
\]
Differentiation gives the maximizing value $t=\beta$, and hence the sufficient heavy-ball boundary $2\mu(1-\beta)^2/[L^2(1+\beta)]$.  Nesterov's fixed-coordinate calculation instead gives $t=\beta^2$ and \eqref{eq:allbeta-bound}.  This is a comparison of restricted certificates, not a comparison of the sharp stability regions; in particular, the denominator $1-\beta+2\beta^2$ is not an intrinsic ``price'' of look-ahead evaluation.
\end{remark}

\subsection{Optimality within a natural quadratic family}

The weight $a_\beta$ in \eqref{eq:V} is not an arbitrary convenient choice.  Keep the slow coordinate \eqref{eq:u}, but consider
\begin{equation}
 V_k(a)=\abs{u_k}^2+a\abs{s_k}^2,
 \qquad a>0.
 \label{eq:Va}
\end{equation}
Repeating the previous proof gives a positive-definiteness condition whose maximal admissible step size can be optimized analytically.

\begin{proposition}[Best weight in the slow-coordinate family]
\label{prop:optimal-weight}
For $0<\beta<1$, set $d=d_\beta$ and $t=ad^2$.  The inequalities used in the preceding proof certify contraction of the energy \eqref{eq:Va} whenever
\begin{equation}
 0<\gamma<\Gamma_\beta(t)
 :=\frac{2\mu(1+\beta)d^2t}
 {L^2\{t^2+(1-\beta^2+2\beta^3)t+\beta^4\}}.
 \label{eq:Gamma-t}
\end{equation}
The function $t\mapsto\Gamma_\beta(t)$ is uniquely maximized at $t=\beta^2$, equivalently
\begin{equation}
 a=a_\beta=\frac{\beta^2}{(1-\beta)^2}.
\end{equation}
At this value, \eqref{eq:Gamma-t} reduces exactly to \eqref{eq:allbeta-bound}.  This is an optimum of that derived sufficient bound with the slow coordinate fixed, not an optimum over all quadratic forms or the exact contraction region of the chosen form.
\end{proposition}

\begin{proof}
For general $a$, \eqref{eq:V-expansion} gives the matrix coefficients
\begin{align*}
 A(a)&=\frac{2\mu\gamma}{d}-(d^{-2}+a)L^2\gamma^2,\\
 B(a)&=\gamma L\left(\frac{\beta^2}{d^2}+a\beta\right),\\
 D(a)&=a(1-\beta^2).
\end{align*}
The determinant condition $A(a)D(a)-B(a)^2>0$ is equivalent, after substituting $a=t/d^2$, to \eqref{eq:Gamma-t}.  Its derivative has the sign of
\[
 \beta^4-t^2,
\]
because
\[
 \frac{d}{dt}\left\{
 \frac{t}{t^2+ct+\beta^4}\right\}
 =\frac{\beta^4-t^2}{(t^2+ct+\beta^4)^2},
 \qquad c=1-\beta^2+2\beta^3.
\]
Thus the unique maximizer is $t=\beta^2$.  Substitution gives
\[
 \Gamma_\beta(\beta^2)
 =\frac{2\mu(1+\beta)d^2}
 {L^2(1+\beta^2+2\beta^3)}
 =\frac{2\mu d^2}{L^2(d+2\beta^2)},
\]
which is \eqref{eq:allbeta-bound}.
\end{proof}

\begin{corollary}[$L^p$ contraction for $1<p\leq2$ under second moments]
\label{cor:pbelow2}
Under the assumptions and step-size condition of Theorem~\ref{thm:allbeta}, for every $1<p\leq2$,
\[
 \norm{V_k^{1/2}}_p
 \leq \norm{V_k^{1/2}}_2
 \leq r_{\mathrm{MS}}^kV_0^{1/2}.
\]
Thus the all-$\beta$ result also supplies $L^p$ GMC for $p<2$, but it requires the finite second-moment assumptions used in Theorem~\ref{thm:allbeta}.  Theorems~\ref{thm:direct} and \ref{thm:allp} genuinely operate without a second moment; their respective step-size conditions must be checked instead of inferring infinite-variance results from this corollary.
\end{corollary}

\subsection{Finite \texorpdfstring{$p$}{p}th moments at arbitrary fixed momentum}
\label{sec:power-extension}

The Perron condition is useful quantitatively, but its small-momentum restriction is not an impossibility result.  The next theorem uses a power of the quadratic energy and establishes a positive, generally much smaller, step-size interval for every fixed $\beta<1$ and every $p>1$.  In particular, it does not obtain a fractional-moment statement by assuming a second moment.  The proof applies the power remainder inequality at the \emph{zero-step momentum update}, rather than at the unchanged state.  This preserves the strict damping of the increment before controlling the stochastic-gradient remainder.

For $0<\beta<1$, retain $d_\beta,a_\beta,c_\beta,P_\beta$ from \eqref{eq:dbeta}--\eqref{eq:Pbeta}.  Define the positive constants
\begin{align}
 \underline b_{p,\beta}&=\min\{1,\beta^{p-2}\},&
 \overline b_{p,\beta}&=\max\{1,\beta^{p-2}\},&
 \tau_p&=\min\{1,p/2\},\label{eq:power-b}\\
 \mathsf A_p&=\frac{p\underline b_{p,\beta}\mu}{d_\beta},&
 \mathsf D_p&=\tau_p a_\beta(1-\beta^2),&
 \mathsf B_p&=\frac{p\overline b_{p,\beta}\beta^2(1+\beta)L_p}{d_\beta^2},\label{eq:power-ABD}\\
 \mathfrak d_p&=\frac{\min\{\mathsf A_p,\mathsf D_p\}}
 {2\lambda_{\max}(P_\beta)},&
 H_p&=\frac{(1+\beta^2)L_p}{d_\beta}.&&\label{eq:power-delta}
\end{align}
Here $\mathfrak d_p$ denotes a deterministic drift constant.  Set
\begin{equation}
 C_p=\begin{cases}
  2^{2-p}H_p^p,&1<p<2,\\
  p(p-1)2^{p-3}H_p^2,&p\geq2,
 \end{cases}
 \qquad
 \varepsilon_p=\begin{cases}
  \{\mathfrak d_p/(2C_p)\}^{1/(p-1)},&1<p<2,\\
  \min\{H_p^{-1},\mathfrak d_p/(2C_p)\},&p\geq2.
 \end{cases}
 \label{eq:power-C}
\end{equation}
All of these constants may depend on the fixed momentum parameter.

\begin{theorem}[All fixed momentum values under a finite $p$th moment]
\label{thm:allp}
Suppose Assumptions~\ref{ass:strong}--\ref{ass:lp} hold for some $p>1$.  Let $0<\beta<1$ and
\begin{equation}
 0<\gamma<\gamma_{p,\beta}^{\mathrm{pow}}
 :=\min\left\{1,\mathfrak d_p^{-1},
       \frac{\mathsf A_p\mathsf D_p}{\mathsf B_p^2},\varepsilon_p\right\}.
 \label{eq:allp-step}
\end{equation}
For the synchronous energy $V_k$ in \eqref{eq:V},
\begin{equation}
 \E(V_{k+1}^{p/2}\mid\cF_k)
 \leq (1-\mathfrak d_p\gamma/2)V_k^{p/2}.
 \label{eq:allp-contraction}
\end{equation}
Consequently the augmented recursion is $L^p$ GMC with rate
$r_{p,\beta}^{\mathrm{pow}}=(1-\mathfrak d_p\gamma/2)^{1/p}<1$ and has a unique causal stationary solution with finite $p$th moment.  For $\beta=0$, Proposition~\ref{prop:base} applies instead.
\end{theorem}

\begin{proof}
We condition throughout on $\cF_k$ and omit the index $k$.  Put
\[
 w=(u,\sqrt{a_\beta}s),\qquad
 D_0=\diag(I_d,\beta I_d),\qquad
 T_0=\begin{pmatrix}d_\beta^{-1}I_d\\ \sqrt{a_\beta}I_d\end{pmatrix}.
\]
The exact recursion is $w_+=D_0w-\gamma T_0h$, and $|w|^2=V$.  Let $W=|D_0w|^2=V-a_\beta(1-\beta^2)|s|^2$.  Since $\beta^2V\leq W\leq V$, elementary integration of $t^{p/2-1}$ on $[W,V]$ gives
\begin{equation}
 W^{p/2}\leq V^{p/2}
 -\tau_p a_\beta(1-\beta^2)|s|^2 V^{p/2-1}.
 \label{eq:power-zero-step}
\end{equation}
Indeed, $1-t^{p/2}\geq\min\{1,p/2\}(1-t)$ for $0\leq t\leq1$.
For $V>0$,
\[
 \underline b_{p,\beta}V^{p/2-1}
 \leq W^{p/2-1}\leq
 \overline b_{p,\beta}V^{p/2-1}.
\]
Writing $\bar h=\E(h\mid\cF_k)$, the inner product in the first-order Taylor term is
\begin{align*}
 \langle D_0w,T_0\bar h\rangle
 &=d_\beta^{-1}\langle u,\bar h\rangle
   +a_\beta\beta\langle s,\bar h\rangle\\
 &=d_\beta^{-1}\langle y,\bar h\rangle
   +\frac{\beta^2(1+\beta)}{d_\beta^2}\langle s,\bar h\rangle.
\end{align*}
Mean monotonicity and $|\bar h|\leq L_p|y|$ thus imply
\begin{align}
 -p\gamma W^{p/2-1}\langle D_0w,T_0\bar h\rangle
 &\leq \gamma V^{p/2-1}
       \{-\mathsf A_p|y|^2+\mathsf B_p|y||s|\}\notag\\
 &\leq \gamma V^{p/2-1}
       \left\{-\frac{\mathsf A_p}{2}|y|^2
       +\frac{\mathsf B_p^2}{2\mathsf A_p}|s|^2\right\}.
 \label{eq:power-first-order}
\end{align}
Combine \eqref{eq:power-zero-step}--\eqref{eq:power-first-order}.  Since
$\gamma\leq\mathsf A_p\mathsf D_p/\mathsf B_p^2$ and $\gamma\leq1$, the sum of the zero-order and first-order terms is at most
\[
 V^{p/2}-\frac{\gamma}{2}V^{p/2-1}
       \{\mathsf A_p|y|^2+\mathsf D_p|s|^2\}
 \leq (1-\mathfrak d_p\gamma)V^{p/2}.
\]
The last step uses $V\leq\lambda_{\max}(P_\beta)(|y|^2+|s|^2)$.

It remains to bound the power remainder.  Since
\[
 |y|=|u-c_\beta s|\leq\sqrt{1+\beta^2}\,V^{1/2},
 \qquad
 \|T_0\|_{\mathrm{op}}=\sqrt{1+\beta^2}/d_\beta,
\]
we have $\|T_0h\|_p\leq H_pV^{1/2}$ conditionally.  For $1<p<2$, Lemma~\ref{lem:power}(ii) bounds the conditional remainder by $C_p\gamma^pV^{p/2}$.  This calculation uses no second moment.  For $p\geq2$, the scalar remainder $(a+b)^p-a^p-pa^{p-1}b$ is nondecreasing in $a,b\geq0$, as is seen from its integral second-derivative formula.  Hence $|D_0w|\leq V^{1/2}$, Lemma~\ref{lem:power}(i), and the Jensen remainder bound proved in \Cref{app:base} give
\[
 \E(\text{remainder}\mid\cF_k)
 \leq\{(1+\gamma H_p)^p-1-p\gamma H_p\}V^{p/2}
 \leq C_p\gamma^2V^{p/2},
\]
where the final inequality follows by integrating the second derivative on $[0,\gamma H_p]\subseteq[0,1]$.  Thus the remainder is at most $(\mathfrak d_p\gamma/2)V^{p/2}$ under \eqref{eq:allp-step}, proving \eqref{eq:allp-contraction}.  If $V=0$, stochastic Lipschitz continuity gives $h=0$ almost surely, and the same conclusion is immediate.  Taking $p$th roots and using the finite $p$th anchor displacement proves stationarity and GMC by backward iteration.
\end{proof}

\begin{remark}[What the power argument does and does not improve]
\label{rem:power-scope}
Theorem~\ref{thm:allp} removes the qualitative small-$\beta$ obstruction for both $1<p<2$ and $p>2$ under mean-only curvature.  It does \emph{not} make the useful mean-square interval in Theorem~\ref{thm:allbeta} an $L^p$ interval automatically.  The new bound is substantially smaller: for $(\mu,L_p)=(1,1.5)$ and $p=1.4$, \eqref{eq:allp-step} is approximately $1.51\times10^{-8}$ at $\beta=0.8$ and $1.54\times10^{-10}$ at $\beta=0.9$.  It is therefore a finite-moment existence and coupling theorem, not a practical tuning rule.  Obtaining useful all-momentum fractional-moment regions remains a substantive quantitative problem.  When its condition holds, the direct Perron theorem can be far more informative.
\end{remark}

\section{Verification in statistical learning models}
\label{sec:models}

The assumptions are stated directly in terms of the stochastic gradient and can be checked without assuming bounded gradients.  The following examples also show where the fractional-moment theorem is useful.

\subsection{Random-design least squares}

Let $X=(Z,R)$ with $Z\in\R^d$ and
\begin{equation}
 R=Z^\trans\theta^\star+\eps,
 \qquad \E(\eps\mid Z)=0,
 \label{eq:lin-model}
\end{equation}
and use the squared loss $g(\theta,(Z,R))=\tfrac12(R-Z^\trans\theta)^2$.  Then
\begin{equation}
 G(\theta,X)=ZZ^\trans(\theta-\theta^\star)-Z\eps,
 \qquad
 m(\theta)=Q(\theta-\theta^\star),
 \quad Q=\E(ZZ^\trans).
 \label{eq:lin-gradient}
\end{equation}
If $Q\succeq\mu I_d$ and
\begin{equation}
 L_p^{\mathrm{LS}}
 :=\sup_{\abs v=1}\norm{ZZ^\trans v}_p<\infty,
 \qquad
 \norm{Z\eps}_p<\infty,
 \label{eq:ls-constants}
\end{equation}
then Assumptions~\ref{ass:strong}--\ref{ass:lp} hold with $L_p=L_p^{\mathrm{LS}}$ and $\sigma_p=\norm{Z\eps}_p$.  Hence Theorem~\ref{thm:direct} applies whenever
\begin{equation}
 \beta\gamma L_p^{\mathrm{LS}}
 <(1-\beta)\{1-q_{\gamma,p}(\mu,L_p^{\mathrm{LS}})\}.
 \label{eq:ls-direct}
\end{equation}
For a standard Gaussian design $Z\sim N(0,I_d)$, $\mu=1$ and rotational symmetry gives, for every unit vector $v$,
\[
 \E|ZZ^\trans v|^2=\E\{|Z|^2(Z^\trans v)^2\}
 =\E Z_1^4+\sum_{j=2}^d\E(Z_1^2Z_j^2)=d+2.
\]
Thus $L_2^{\rm LS}=\sqrt{d+2}$ and the direct mean-square criterion has the small-step momentum range
$\beta<1/(1+\sqrt{d+2})$, approximately $0.333$, $0.224$, and $0.122$ at $d=2,10,50$.  It does not cover momenta $0.8$--$0.95$ in these examples.  The all-momentum quadratic or power theorem requires its own smaller step-size check.  For $1<p<2$, condition \eqref{eq:ls-constants} permits $Z\eps$ to have infinite variance.  If, more strongly, $\abs Z\leq R_Z$ almost surely and ridge regularization $\lambda\abs\theta^2/2$ is added, then the sample Hessian lies in the sector
\[
 \lambda I_d\preceq ZZ^\trans+\lambda I_d
 \preceq(R_Z^2+\lambda)I_d,
\]
so Corollary~\ref{cor:pathwise} and the endpoint LMI of Theorem~\ref{thm:lmi} are available.

\subsection{Ridge-regularized logistic regression}

Let $Y\in\{0,1\}$, $Z\in\R^d$, $\sigma(t)=(1+e^{-t})^{-1}$, and
\begin{equation}
 g(\theta,(Z,Y))
 =-Y Z^\trans\theta+\log(1+e^{Z^\trans\theta})
 +\frac{\lambda}{2}\abs\theta^2,
 \qquad \lambda>0.
 \label{eq:logistic-loss}
\end{equation}
The stochastic gradient and Hessian are
\begin{align}
 G(\theta,X)&=\{\sigma(Z^\trans\theta)-Y\}Z+\lambda\theta,
 \label{eq:logistic-gradient}\\
 \nabla_\theta G(\theta,X)&=
 \sigma(Z^\trans\theta)\{1-\sigma(Z^\trans\theta)\}ZZ^\trans+\lambda I_d.
 \label{eq:logistic-hessian}
\end{align}
Consequently the mean field is $\lambda$-strongly monotone, and
\begin{equation}
 \norm{G(\theta,X)-G(\vartheta,X)}_p
 \leq\left\{\lambda+\frac14\norm{\abs Z^2}_p\right\}
 \abs{\theta-\vartheta}.
 \label{eq:logistic-Lp}
\end{equation}
Thus one may take
\[
 \mu=\lambda,
 \qquad
 L_p=\lambda+\frac14\norm{\abs Z^2}_p.
\]
At the population minimizer, $\abs{\sigma(Z^\trans\theta^\star)-Y}\leq1$, so $\sigma_p<\infty$ follows from $\norm{\abs Z}_p<\infty$.  If $\abs Z\leq R_Z$ almost surely, \eqref{eq:logistic-hessian} gives the samplewise sector
\begin{equation}
 \lambda I_d\preceq\nabla_\theta G(\theta,X)
 \preceq(\lambda+R_Z^2/4)I_d.
 \label{eq:logistic-sector}
\end{equation}
The all-$\beta$ and pathwise GMC theorems therefore give explicit stability regions for constant-parameter stochastic Nesterov logistic regression.  The ridge term is essential for global strong monotonicity; without it, logistic risk may be only locally strongly convex \citep{Bach2014}.

\subsection{Additive heavy-tailed gradient noise}

Suppose
\begin{equation}
 G(\theta,X)=\nabla G_0(\theta)-\xi,
 \qquad \E\xi=0,
 \label{eq:additive-general}
\end{equation}
where $G_0$ is $\mu$-strongly convex with $L$-Lipschitz gradient and $\norm{\xi}_p<\infty$ for some $p>1$.  Then $L_p=L$ and $\sigma_p=\norm{\xi}_p$.  The direct GMC condition becomes
\begin{equation}
 \beta\gamma L<(1-\beta)\{1-q_{\gamma,p}(\mu,L)\}.
 \label{eq:additive-direct}
\end{equation}
No variance is needed when $1<p<2$.  This setting separates two different sources of heavy tails: the innovations may themselves be regularly varying, while the deterministic Jacobian controls memory.  In contrast, multiplicative-noise least squares can generate heavy stationary tails even from light-tailed additive errors \citep{GurbuzbalabanSimsekliZhu2021,HodgkinsonMahoney2021,DamekMentemeier2026}.

\section{Pathwise sector conditions and an endpoint LMI}
\label{sec:pathwise}

The preceding theorems assume curvature only after averaging over $X$.  Stronger samplewise curvature yields stronger conclusions and a useful computational refinement.

\begin{assumption}[Samplewise Hessian sector]
\label{ass:sector}
For almost every $x$, $g(\cdot,x)$ is twice continuously differentiable and there are constants $0<\mu_s\leq L_s<\infty$ such that
\begin{equation}
 \mu_s I_d\preceq\nabla^2 g(\theta,x)\preceq L_s I_d,
 \qquad \theta\in\R^d.
 \label{eq:sector}
\end{equation}
\end{assumption}

\begin{corollary}[Pathwise all-$p$ GMC]
\label{cor:pathwise}
Under Assumption~\ref{ass:sector}, let $0<\beta<1$ and
\begin{equation}
 0<\gamma<
 \frac{2\mu_s(1-\beta)^2}
 {L_s^2\{1-\beta+2\beta^2\}}.
 \label{eq:pathwise-bound}
\end{equation}
Then there are $P_\beta\succ0$ and $r<1$ such that, for almost every realization $x$ and every pair of states,
\begin{equation}
 V\{\Phi_x(z)-\Phi_x(\widetilde z)\}
 \leq r^2V(z-\widetilde z),
 \label{eq:pathwise-contract}
\end{equation}
where $V$ is the quadratic form in \eqref{eq:V}.  If $\norm{G(\theta^\star,X_0)}_p<\infty$ for some $p>0$, the stationary state has a finite $p$th moment and the pathwise contraction implies $L^p$ GMC with rate $r$.
\end{corollary}

\begin{proof}
By the mean-value theorem,
\[
 h=H(y-\widetilde y),
 \qquad
 H=\int_0^1\nabla^2g\{\widetilde y+t(y-\widetilde y),x\}\,dt,
\]
with $\mu_sI\preceq H\preceq L_sI$.  Hence
\[
 \ip{y-\widetilde y}{h}\geq\mu_s\abs{y-\widetilde y}^2,
 \qquad
 \abs h\leq L_s\abs{y-\widetilde y}
\]
pointwise.  The proof of Theorem~\ref{thm:allbeta} now contains no conditional expectation and yields \eqref{eq:pathwise-contract}.  Backward iteration and the finite anchor moment give the stationary moment statement for any $p>0$.
\end{proof}

For a less restricted sector-based certificate, diagonalize the integrated Hessian.  The two-endpoint reduction below is the classical polytopic common-Lyapunov vertex test, rather than a new LMI principle; see \citet{BoydEtAl1994,LessardRechtPackard2016}.  In the original two-iterate difference state $(x_k,x_{k-1})$, a curvature eigenmode $h$ evolves through
\begin{equation}
 A_{\gamma,\beta}(h)=
 \begin{pmatrix}
 (1+\beta)(1-\gamma h)&-\beta(1-\gamma h)\\
 1&0
 \end{pmatrix}.
 \label{eq:A-h}
\end{equation}

\begin{theorem}[Two-endpoint common-Lyapunov criterion]
\label{thm:lmi}
Suppose Assumption~\ref{ass:sector} holds.  Assume there exist a symmetric matrix $P_0\succ0$ and $r\in(0,1)$ such that
\begin{equation}
 A_{\gamma,\beta}(\mu_s)^\trans P_0A_{\gamma,\beta}(\mu_s)
 \preceq r^2P_0,
 \qquad
 A_{\gamma,\beta}(L_s)^\trans P_0A_{\gamma,\beta}(L_s)
 \preceq r^2P_0.
 \label{eq:endpoint-lmi}
\end{equation}
Then \eqref{eq:endpoint-lmi} holds for every $h\in[\mu_s,L_s]$.  Moreover, with $P=P_0\otimes I_d$, the $d$-dimensional Nesterov difference map satisfies
\begin{equation}
 \mathsf A(H)^\trans P\mathsf A(H)\preceq r^2P
 \label{eq:matrix-H-lmi}
\end{equation}
for every symmetric $H$ with spectrum in $[\mu_s,L_s]$.  Hence the Nesterov recursion is pathwise GMC in the common quadratic norm.
\end{theorem}

\begin{proof}
Write $A(h)=A_0+hA_1$.  The matrix-valued function
\[
 \mathcal Q(h)=A(h)^\trans P_0A(h)
\]
is convex in the Loewner order because $\mathcal Q''(h)=2A_1^\trans P_0A_1\succeq0$.  For $h=t\mu_s+(1-t)L_s$,
\[
 \mathcal Q(h)\preceq t\mathcal Q(\mu_s)+(1-t)\mathcal Q(L_s)
 \preceq r^2P_0.
\]
For the matrix statement, write $H=U\diag(h_1,\ldots,h_d)U^\trans$.  The orthogonal transformation $I_2\otimes U$ block-diagonalizes the state matrix into the scalar blocks $A(h_j)$, while it leaves $P_0\otimes I_d$ invariant.  Applying the scalar inequality to every block proves \eqref{eq:matrix-H-lmi}.
\end{proof}

\begin{remark}[Relation to control-theoretic analyses]
The endpoint theorem is a deliberately small common-Lyapunov test.  Integral quadratic constraints and dissipativity permit richer multipliers and performance outputs \citep{LessardRechtPackard2016,HuLessard2017}.  Here the purpose is narrower: once a common contraction norm is found, the entire stationary and physical-dependence theory below follows automatically.
\end{remark}

\subsection{A general quadratic certificate using mean-only information}
\label{sec:mean-lmi}

Optimizing the single weight in Proposition~\ref{prop:optimal-weight} does not optimize over all quadratic metrics.  We now allow all three entries of a symmetric $2\times2$ matrix to vary.  The following sufficient condition is a small application of the $S$-procedure; the method itself is classical \citep{BoydEtAl1994,LessardRechtPackard2016}.
Put
\[
 B_\beta=\begin{pmatrix}1&\beta^2\\0&\beta\end{pmatrix},
 \qquad b_{\gamma,\beta}=\gamma\begin{pmatrix}1+\beta\\1\end{pmatrix},
 \qquad e_1=\begin{pmatrix}1\\0\end{pmatrix}.
\]
For $P\succ0$, $\varrho\in(0,1)$, and $z=(y,s)$, define the symmetric $3\times3$ matrix
\begin{equation}
 F(P,\varrho)=
 \begin{pmatrix}
 B_\beta^\trans PB_\beta-\varrho P
  +L^2(b_{\gamma,\beta}^\trans P b_{\gamma,\beta})e_1e_1^\trans
   &-B_\beta^\trans P b_{\gamma,\beta}\\
 -b_{\gamma,\beta}^\trans P B_\beta&0
 \end{pmatrix},
 \label{eq:mean-F}
\end{equation}
with the top-left block of size two.  Let
\[
 Q_1=\begin{pmatrix}-\mu&0&1/2\\0&0&0\\1/2&0&0\end{pmatrix},
 \qquad Q_2=\diag(L^2,0,-1).
\]
\begin{theorem}[Mean-only quadratic $S$-procedure]
\label{thm:mean-lmi}
Under Assumptions~\ref{ass:strong}--\ref{ass:lp} with $p=2$, suppose that for some $P\succ0$, $0<\varrho<1$, and $\lambda_1,\lambda_2\geq0$,
\begin{equation}
 F(P,\varrho)+\lambda_1Q_1+\lambda_2Q_2\preceq0.
 \label{eq:mean-lmi}
\end{equation}
Then the Nesterov state in $(y,s)$ coordinates contracts conditionally in the quadratic metric $P\otimes I_d$ with squared factor $\varrho$.  Hence it is $L^2$ GMC and has a unique causal stationary law with finite second moment.
\end{theorem}
\begin{proof}
The difference update is $z_+=(B_\beta\otimes I_d)z-(b_{\gamma,\beta}\otimes I_d)h$.  Set $\bar h=\E(h\mid\cF_k)$.  Expansion of the quadratic form and $\E(|h|^2\mid\cF_k)\leq L^2|y|^2$ give
\[
 \E(|z_+|_P^2\mid\cF_k)-\varrho|z|_P^2
 \leq (y,s,\bar h)^\trans\{F(P,\varrho)\otimes I_d\}(y,s,\bar h).
\]
The two constraints are
$\langle y,\bar h\rangle-\mu|y|^2\geq0$ and
$L^2|y|^2-|\bar h|^2\geq0$.  Their quadratic matrices are $Q_1\otimes I_d$ and $Q_2\otimes I_d$.  Adding their nonnegative multiples increases the right-hand side, which becomes nonpositive by \eqref{eq:mean-lmi}.  The stationarity argument is unchanged.
\end{proof}

The certificate is sufficient, not an exact characterization of the nonlinear mean-only class.  For fixed $(\gamma,\beta,\varrho)$ it is a semidefinite feasibility problem; a trace normalization $\tr P=1$ removes scaling.  A strict negative slack at $\varrho=1$ also certifies a factor below one: if the left side is at most $-\epsilon I_3$, then $\varrho=1-\epsilon/\lambda_{\max}(P)$ is valid (or any slightly larger positive factor).  Our computations exhibit a feasible certificate at $(\mu,L,\beta,\gamma)=(1,1.5,0.8,0.03)$, beyond the slow-coordinate boundary $0.0240$.  We do not infer infeasibility or instability when a numerical search fails.

\section{Stationary moments and invariant-mean bias}
\label{sec:moments}

The contraction theorems ensure finite stationary moments.  We now identify their small-step scale.  The result is stated for $p=2$ and fixed $\beta<1$; constants may deteriorate as $\beta\uparrow1$.

\begin{theorem}[Stationary mean-square scale with explicit momentum dependence]
\label{thm:scale}
Suppose Assumptions~\ref{ass:strong}--\ref{ass:lp} hold at $p=2$.  For $0<\beta<1$ and $0<\gamma<\gamma_0$ as specified in \eqref{eq:gamma0-scale},
\begin{equation}
 \E|Y_0^\circ-\theta^\star|^2+\E|S_0^\circ|^2
 \leq
 \frac{4(1+\beta^2)}{\min\{\mu,\beta^2(1+\beta)\}}
 \frac{\gamma\sigma_2^2}{1-\beta}.
 \label{eq:stationary-scale}
\end{equation}
For $\beta=0$ and $\gamma\leq\min\{1,\mu/(2L^2)\}$, the left-hand side is at most $(2/\mu+4)\gamma\sigma_2^2$.
In particular, for each fixed $\beta<1$,
\begin{equation}
 \E|\Theta_0^\circ-\theta^\star|^2=O(\gamma).
 \label{eq:theta-scale}
\end{equation}
The range of $\gamma$ in this theorem depends on $\beta$; it is not a fixed-step uniform statement as $\beta\uparrow1$.
\end{theorem}

\begin{proof}
The case $\beta=0$ follows directly from
\[
 \E(\abs{\Theta_{k+1}-\theta^\star}^2\mid\cF_k)
 \leq\{1-2\mu\gamma+2L^2\gamma^2\}
 \abs{\Theta_k-\theta^\star}^2+2\gamma^2\sigma_2^2.
\]
For $\gamma\leq\mu/(2L^2)$, stationarity gives
$\E\abs{\Theta_0^\circ-\theta^\star}^2\leq2\gamma\sigma_2^2/\mu$.
Also $S_{k+1}=-\gamma G(\Theta_k,X_{k+1})$ and
$\E|S_0^\circ|^2\leq\gamma^2\{2L^2\E|\Theta_0^\circ-\theta^\star|^2+2\sigma_2^2\}\leq4\gamma^2\sigma_2^2$ in this range.  Since $Y=\Theta$ at $\beta=0$, the stated combined bound follows.

Assume henceforth that $0<\beta<1$ and abbreviate
\begin{equation}
 d=1-\beta,
 \quad N_\beta=\frac{1+\beta^2}{d^2},
 \quad J_\beta=\frac{\beta^2(1+\beta)}{d^2},
 \quad D_\beta=\frac{\beta^2(1+\beta)}{d}.
 \label{eq:scale-constants}
\end{equation}
Apply the energy \eqref{eq:V} to one trajectory relative to $\theta^\star$:
\[
 x_k=\Theta_k-\theta^\star,
 \quad y_k=Y_k-\theta^\star,
 \quad s_k=S_k,
 \quad u_k=x_k+\frac{\beta}{d}s_k.
\]
Put $g_{k+1}=G(\theta^\star+y_k,X_{k+1})$.  Conditional on $\cF_k$,
\begin{align}
 \ip{y_k}{\E(g_{k+1}\mid\cF_k)}&\geq\mu\abs{y_k}^2,
 \label{eq:scale-mono}\\
 \abs{\E(g_{k+1}\mid\cF_k)}&\leq L\abs{y_k},
 \label{eq:scale-meanlip}\\
 \E(\abs{g_{k+1}}^2\mid\cF_k)
 &\leq2L^2\abs{y_k}^2+2\sigma_2^2.
 \label{eq:scale-g2}
\end{align}
The exact energy expansion used in \eqref{eq:V-expansion} now yields
\begin{align}
 \E(V_{k+1}\mid\cF_k)
 &\leq V_k-\frac{2\mu\gamma}{d}\abs{y_k}^2
 +2J_\beta L\gamma\abs{y_k}\abs{s_k}
 +2N_\beta L^2\gamma^2\abs{y_k}^2\notag\\
 &\qquad-D_\beta\abs{s_k}^2
 +2N_\beta\gamma^2\sigma_2^2.
 \label{eq:scale-exact-drift}
\end{align}
By Young's inequality,
\begin{equation}
 2J_\beta L\gamma\abs y\abs s
 \leq\frac{\mu\gamma}{d}\abs y^2
 +\frac{J_\beta^2L^2d}{\mu}\gamma\abs s^2.
 \label{eq:scale-young}
\end{equation}
Choose
\begin{equation}
 \gamma_0=\min\left\{
 \Gamma_{\mathrm{NAG}}(\beta),
 \frac{\mu}{4N_\beta L^2d},
 \frac{D_\beta\mu}{2J_\beta^2L^2d},
 1\right\}.
 \label{eq:gamma0-scale}
\end{equation}
For $0<\gamma<\gamma_0$, \eqref{eq:scale-exact-drift}--\eqref{eq:scale-young} imply
\begin{equation}
 \E(V_{k+1}\mid\cF_k)
 \leq V_k-\frac{\mu\gamma}{2d}\abs{y_k}^2
 -\frac{D_\beta}{2}\abs{s_k}^2
 +2N_\beta\gamma^2\sigma_2^2.
 \label{eq:scale-final-drift}
\end{equation}
The harmless restriction $\gamma\leq1$ in \eqref{eq:gamma0-scale} is used only to weaken the increment drift to order $\gamma$.  With
\[
 d_\beta^{\rm drift}=\min\left\{\frac{\mu}{2d},\frac{D_\beta}{2}\right\}>0,
\]
we may weaken \eqref{eq:scale-final-drift} to
\[
 \E(V_{k+1}\mid\cF_k)
 \leq V_k-d_\beta^{\rm drift}\gamma(\abs{y_k}^2+\abs{s_k}^2)
 +2N_\beta\gamma^2\sigma_2^2.
\]
At stationarity the expected energy increment is zero; hence
\begin{equation}
 \E\abs{Y_0^\circ-\theta^\star}^2+
 \E\abs{S_0^\circ}^2
 \leq\frac{2N_\beta}{d_\beta^{\rm drift}}\gamma\sigma_2^2.
 \label{eq:scale-explicit}
\end{equation}
Substituting $N_\beta=(1+\beta^2)/(1-\beta)^2$ and $d_\beta^{\rm drift}=\min\{\mu,\beta^2(1+\beta)\}/\{2(1-\beta)\}$ yields \eqref{eq:stationary-scale}.  Finally, $\Theta_0^\circ-\theta^\star=(Y_0^\circ-\theta^\star)-\beta S_0^\circ$, so \eqref{eq:theta-scale} follows.
\end{proof}

Let
\begin{equation}
 \theta_\infty(\gamma,\beta)=\E\Theta_0^\circ.
 \label{eq:stationary-mean}
\end{equation}
Because $S_k^\circ=\Theta_k^\circ-\Theta_{k-1}^\circ$ is stationary and integrable, $\E S_k^\circ=0$ and hence $\E Y_k^\circ=\E\Theta_k^\circ$.  Taking expectations in
\[
 S_{k+1}^\circ=\beta S_k^\circ-\gamma G(Y_k^\circ,X_{k+1})
\]
gives the exact stationary score equation
\begin{equation}
 \E m(Y_0^\circ)=0.
 \label{eq:stationary-score}
\end{equation}

\begin{corollary}[Invariant-mean bias]
\label{cor:bias}
In addition to the assumptions of Theorem~\ref{thm:scale}, suppose $m$ is continuously differentiable, $H_\star=\nabla m(\theta^\star)$ is nonsingular, and its Jacobian is globally Lipschitz: for some $M<\infty$,
\begin{equation}
 \norm{\nabla m(\theta)-\nabla m(\vartheta)}_{\mathrm{op}}
 \leq M\abs{\theta-\vartheta},
 \qquad \theta,\vartheta\in\R^d.
 \label{eq:taylor-m}
\end{equation}
For $0<\beta<1$ in the range of Theorem~\ref{thm:scale}, the proof gives the explicit bound
\[
 |\theta_\infty(\gamma,\beta)-\theta^\star|
 \leq \frac{2M\|H_\star^{-1}\|(1+\beta^2)}{\min\{\mu,\beta^2(1+\beta)\}}
 \frac{\gamma\sigma_2^2}{1-\beta}.
\]
Consequently, for fixed $\beta<1$,
\begin{equation}
 \abs{\theta_\infty(\gamma,\beta)-\theta^\star}=O(\gamma),
 \qquad \gamma\downarrow0.
 \label{eq:bias-order}
\end{equation}
\end{corollary}

\begin{proof}
Equation \eqref{eq:stationary-score} and the integral Taylor formula give
\[
 0=H_\star\{\E Y_0^\circ-\theta^\star\}+\E R(Y_0^\circ),
 \qquad
 \abs{R(y)}\leq \frac{M}{2}\abs{y-\theta^\star}^2.
\]
Since $\E Y_0^\circ=\theta_\infty(\gamma,\beta)$,
\[
 \abs{\theta_\infty(\gamma,\beta)-\theta^\star}
 \leq\frac{M}{2}\norm{H_\star^{-1}}\E\abs{Y_0^\circ-\theta^\star}^2
 =O(\gamma)
\]
by Theorem~\ref{thm:scale}.
\end{proof}

\begin{remark}
The center of the fixed-step averaged-iterate CLT is $\theta_\infty(\gamma,\beta)$, not $\theta^\star$.  Corollary~\ref{cor:bias} separates the statistical $n^{-1/2}$ fluctuation from the algorithmic $O(\gamma)$ invariant bias.  Multi-step-size extrapolation and invariant-measure expansions, as developed for ordinary SGD and numerical ergodic dynamics by \citet{AbdulleVilmartZygalakis2014,DieuleveutDurmusBach2020,LiLouRichterWu2026}, are a natural next step for stochastic Nesterov acceleration.
\end{remark}

\section{Physical dependence and initialization-robust limit theory}
\label{sec:limits}

Let $Z_k^\circ=(Y_k^\circ,S_k^\circ)$ be the two-sided stationary causal solution constructed above.  There is a measurable map $H$ such that
\begin{equation}
 Z_k^\circ=H(\ldots,X_{k-1},X_k).
 \label{eq:causal}
\end{equation}
Let $X_0'$ be an independent copy of $X_0$, independent of the full innovation sequence, and define
\[
 \cF_k=(\ldots,X_{-1},X_0,X_1,\ldots,X_k),
 \qquad
 \cF_k^{\{0\}}=(\ldots,X_{-1},X_0',X_1,\ldots,X_k).
\]
Write $Z_k^{\circ,\{0\}}=H(\cF_k^{\{0\}})$ and
\begin{equation}
 \delta_{k,p}^{Z}=\norm{Z_k^\circ-Z_k^{\circ,\{0\}}}_p,
 \qquad
 \delta_{k,p}^{\Theta}=\norm{\Theta_k^\circ-\Theta_k^{\circ,\{0\}}}_p.
 \label{eq:physical}
\end{equation}
Any fixed norm on $\R^{2d}$ may be used in $\delta_{k,p}^Z$.

\begin{proposition}[Exponential physical dependence]
\label{prop:physical}
Assume one of the following:
\begin{enumerate}[label=\textup{(\alph*)},leftmargin=2.2em]
\item the conditions of Proposition~\ref{prop:ar2}, Theorem~\ref{thm:direct}, or Theorem~\ref{thm:allp} hold at order $p>1$;
\item the conditions of Theorem~\ref{thm:allbeta} or Theorem~\ref{thm:mean-lmi} hold and $p=2$;
\item Corollary~\ref{cor:pathwise} or Theorem~\ref{thm:lmi} supplies pathwise contraction and $\norm{G(\theta^\star,X_0)}_p<\infty$.
\end{enumerate}
Then there are $C<\infty$ and $r\in(0,1)$ such that
\begin{equation}
 \delta_{k,p}^{Z}+\delta_{k,p}^{\Theta}\leq Cr^k,
 \qquad k\geq0.
 \label{eq:delta-exp}
\end{equation}
Consequently the cumulative dependence tail
\begin{equation}
 \Theta_{m,p}=\sum_{k=m}^{\infty}\delta_{k,p}^{\Theta}
 \label{eq:Theta-tail}
\end{equation}
satisfies $\Theta_{m,p}\leq Cr^m/(1-r)$.
\end{proposition}

\begin{proof}
The innovation replacement changes the state at time zero by a random amount with finite $p$th moment because it changes only one application of $\Phi_X$ and the stationary state has a finite $p$th moment.  From time one onward, the original and coupled processes use identical innovations.  Applying the appropriate synchronous contraction for $k$ subsequent steps gives $\delta_{k,p}^Z\leq Cr^k$.  Equation \eqref{eq:inverse-state} transfers this bound to $\Theta_k$.  Summing the geometric series proves the last assertion.
\end{proof}

Put
\begin{equation}
 \bar\Theta_n=\frac1n\sum_{k=1}^n\Theta_k,
 \qquad
 \bar\Theta_n^\circ=\frac1n\sum_{k=1}^n\Theta_k^\circ,
 \qquad
 \theta_\infty=\E\Theta_0^\circ.
 \label{eq:averages}
\end{equation}

\begin{theorem}[Moment and almost-sure rates for averaged Nesterov iterates]
\label{thm:average-rates}
Under the conditions of Proposition~\ref{prop:physical} for some $p>1$,
\begin{equation}
 \norm{\bar\Theta_n-\theta_\infty}_p
 =O\left(n^{1/(p\wedge2)-1}\right).
 \label{eq:moment-rate}
\end{equation}
The same rate holds for $\bar\Theta_n^\circ$.  Moreover,
\begin{align}
 \abs{\bar\Theta_n-\theta_\infty}
 &=o_{\as}(n^{1/p-1}),&&1<p<2,
 \label{eq:as-pbelow2}\\
 \abs{\bar\Theta_n-\theta_\infty}
 &=o_{\as}(1),&&p=2,
 \label{eq:as-p2}\\
 \abs{\bar\Theta_n-\theta_\infty}
 &=O_{\as}\{(n^{-1}\log\log n)^{1/2}\},&&p>2.
 \label{eq:as-pabove2}
\end{align}
The last display is understood for $n\geq3$.
\end{theorem}

\begin{proof}
By Proposition~\ref{prop:physical}, the physical-dependence tail is summable and in fact exponential.  Theorem~1 of \citet{Wu2007}, applied coordinatewise, gives
\[
 \norm{\sum_{k=1}^n(\Theta_k^\circ-\theta_\infty)}_p
 =O\{n^{1/(p\wedge2)}\},
\]
which proves the stationary version of \eqref{eq:moment-rate}.  Corollary~2(iii) of \citet{Wu2007} gives \eqref{eq:as-pbelow2}: its projective-tail condition (9) follows from exponential physical dependence.  The ergodic strong law gives \eqref{eq:as-p2}.  Theorem~2(i) of the same paper gives the coordinatewise law-of-the-iterated-logarithm bound \eqref{eq:as-pabove2}; exponential tails satisfy its projective summability condition.  These are applications of those probability theorems, not new general strong laws.

It remains to remove stationary initialization.  Under any of the contraction hypotheses, coordinate equivalence and synchronous coupling imply
\begin{equation}
 \norm{\Theta_k-\Theta_k^\circ}_p\leq C_0r^k.
 \label{eq:init-theta}
\end{equation}
Therefore
\begin{equation}
 \norm{\sum_{k=1}^n(\Theta_k-\Theta_k^\circ)}_p
 \leq\sum_{k=1}^{\infty}C_0r^k<\infty,
 \label{eq:init-Op1}
\end{equation}
uniformly in $n$.  Moreover, $\sum_{k\geq1}\E\abs{\Theta_k-\Theta_k^\circ}<\infty$, so Tonelli's theorem gives $\sum_{k\geq1}\abs{\Theta_k-\Theta_k^\circ}<\infty$ almost surely.  Dividing by $n$ transfers the stationary moment and almost-sure results to every deterministic initial pair.
\end{proof}

For $p=2$, define the long-run covariance
\begin{equation}
 \Sigma_{\gamma,\beta}
 =\sum_{j\in\mathbb Z}
 \Cov(\Theta_0^\circ,\Theta_j^\circ).
 \label{eq:lrcov}
\end{equation}
The sum is absolutely convergent entrywise under exponential physical dependence.

\begin{theorem}[Initialization-robust central limit theorem]
\label{thm:clt}
Suppose the conditions of Proposition~\ref{prop:physical} hold with $p=2$.  Then
\begin{equation}
 n^{-1/2}\sum_{k=1}^n(\Theta_k^\circ-\theta_\infty)
 \ \Rightarrow\ N(0,\Sigma_{\gamma,\beta}).
 \label{eq:stationary-clt}
\end{equation}
For every deterministic initial pair $(\Theta_{-1},\Theta_0)\in\R^{2d}$,
\begin{equation}
 \sqrt n(\bar\Theta_n-\theta_\infty)
 \ \Rightarrow\ N(0,\Sigma_{\gamma,\beta}).
 \label{eq:quenched-clt}
\end{equation}
More generally, the joint CLT holds for any finite collection of step-size/momentum pairs driven by the same data stream, provided each pair satisfies its contraction condition.
\end{theorem}

\begin{proof}
Exponential physical dependence verifies the projective summability criterion for Theorem~3 of \citet{Wu2011}; see equations~(44) and (48) there for the projection and long-run covariance formulations.  Applying the Cram\'er--Wold device proves \eqref{eq:stationary-clt}.  Equation \eqref{eq:init-Op1} shows that the difference between the stationary and arbitrarily initialized partial sums is $O_{L^2}(1)$, hence $o_P(n^{1/2})$, proving \eqref{eq:quenched-clt}.  Stacking finitely many parameter pairs preserves exponential physical dependence and gives the joint result.
\end{proof}

\begin{remark}[Use of the word ``quenched'']
Classical quenched limit theory conditions on a random environment or on the infinite past.  Here the practically important statement is slightly different: the same limit holds under the law $\Pp_z$ for every deterministic initial state $z=(\Theta_{-1},\Theta_0)$.  We therefore use ``initial-state quenched'' only as shorthand and state the deterministic-initialization property explicitly.
Theorem~5 of \citet{ChenWu2016} already treats a zero-past initialized autoregression, centered by its time-dependent expectations.  Here \eqref{eq:init-theta} additionally gives
\[
 \sum_{k\geq1}|\E\Theta_k-\theta_\infty|
 \leq\sum_{k\geq1}\|\Theta_k-\Theta_k^\circ\|_1<\infty,
\]
which justifies stationary-mean centering for every deterministic initial pair.  This is an explicit consequence of geometric forgetting, not a new general initialization principle.
\end{remark}

\begin{theorem}[Strong Gaussian approximation]
\label{thm:sip}
Suppose Proposition~\ref{prop:ar2}, Theorem~\ref{thm:direct}, or Theorem~\ref{thm:allp} supplies order-$p$ GMC for some $p>2$, or Corollary~\ref{cor:pathwise} or Theorem~\ref{thm:lmi} supplies pathwise contraction with a finite $p$th anchor moment.  Assume $\Sigma_{\gamma,\beta}$ is positive definite.  On an enriched probability space there exist random vectors with the same finite-dimensional distributions as the stationary centered Nesterov process and i.i.d. $N(0,I_d)$ vectors $(Z_k)$ such that
\begin{equation}
 \max_{1\leq j\leq n}
 \left\lvert
 \sum_{k=1}^j(\Theta_k^\circ-\theta_\infty)
 -\Sigma_{\gamma,\beta}^{1/2}\sum_{k=1}^j Z_k
 \right\rvert
 =o_P(n^{1/p}).
 \label{eq:sip}
\end{equation}
The same approximation rate holds for the recursion started from every deterministic initial pair.
\end{theorem}

\begin{proof}
The coupling in \eqref{eq:sip} is the dependent-process analogue of the Koml\'os--Major--Tusn\'ady embedding \citep{KomlosMajorTusnady1975,KomlosMajorTusnady1976,BerkesLiuWu2014}.  Equation \eqref{eq:delta-exp} implies that for every $A>0$ and every $\chi>0$,
\[
 \Theta_{m,p}=O\{m^{-\chi}(\log m)^{-A}\}.
\]
We verify the remaining hypotheses of Theorem~2.2(i) of \citet{KarmakarWu2020}, using its stationary specialization (as in equation~(3.4) there).  The finite $p$th moment and stationarity imply uniform integrability of the $p$th powers, their condition (2.A).  Write $C_j=\Cov(\Theta_0^\circ,\Theta_j^\circ)$.  Exponential physical dependence and the projection decomposition imply $\sum_{j\in\mathbb Z}|j|\|C_j\|<\infty$.  Consequently
\[
 \left\|\Var\left(\sum_{k=t+1}^{t+\ell}(\Theta_k^\circ-\theta_\infty)\right)
           -\ell\Sigma_{\gamma,\beta}\right\|
 \leq 2\sum_{j\in\mathbb Z}|j|\|C_j\|=O(1),
\]
uniformly in $t$.  Since $\Sigma_{\gamma,\beta}\succ0$, their block covariance lower bound (2.B) holds for all sufficiently large $\ell$.  Exponential tails allow any $\chi$ above their critical exponent and any positive logarithmic exponent in condition (2.3).  Thus Theorem~2.2(i) gives \eqref{eq:sip} for the stationary process.  Its conclusion is the in-probability approximation (1.4), not an almost-sure $o(n^{1/p})$ rate.  For an arbitrary initialization,
\begin{align*}
 \norm{\max_{j\leq n}
 \abs{\sum_{k=1}^j(\Theta_k-\Theta_k^\circ)}}_p
 &\leq\sum_{k=1}^n\norm{\Theta_k-\Theta_k^\circ}_p\\
 &\leq C\sum_{k=1}^{\infty}r^k=O(1),
\end{align*}
which is $o(n^{1/p})$.  Adding this discrepancy to the stationary coupling proves the claim.
\end{proof}

\begin{remark}[Moment order and the meaning of ``strong'']
\label{rem:sip-scope}
The all-momentum mean-square theorem is sufficient for the CLT, but by itself is not sufficient for an order-$p>2$ Gaussian approximation.  Theorem~\ref{thm:allp} supplies such an order at each fixed momentum after reducing the step size according to \eqref{eq:allp-step}; it does not upgrade an already selected $L^2$ step.  Alternatively, a verified samplewise common norm permits every finite anchor moment at the same certified step.  ``Strong Gaussian approximation'' here refers to a coupling on an enriched space with the displayed uniform error in probability.  We do not assert an almost-sure rate from the cited multivariate theorem. Theorem~6 of \citet{ChenWu2016} already gives a scalar almost-sure $o(n^{1/p})$ approximation under its autoregressive lag conditions.  In dimension one those conditions are satisfied by the finite-lag criterion \eqref{eq:ar2-condition}; that scalar result is not a multivariate almost-sure theorem.
\end{remark}

\begin{corollary}[Functional central limit theorem]
Under the assumptions of Theorem~\ref{thm:sip}, for every deterministic initialization,
\begin{equation}
 \left\{n^{-1/2}\sum_{k=1}^{\lfloor nt\rfloor}
 (\Theta_k-\theta_\infty):0\leq t\leq1\right\}
 \Rightarrow
 \left\{\Sigma_{\gamma,\beta}^{1/2}B(t):0\leq t\leq1\right\}
 \label{eq:fclt}
\end{equation}
in $D([0,1],\R^d)$, where $B$ is standard $d$-dimensional Brownian motion.
\end{corollary}

\section{Exact additive-quadratic theory}
\label{sec:quadratic}

The global contraction conditions are sufficient, not necessary.  The classical linear-systems description of Nesterov on quadratics provides a benchmark for their conservatism; see \citet{CanGurbuzbalabanZhu2019,AssranRabbat2020}.  The calculations below are included as explicit benchmarks and inferential interpretations, not as new general AR(2) theory.

Assume
\begin{equation}
 G(\theta,X_k)=H(\theta-\theta^\star)-\xi_k,
 \label{eq:quadratic-gradient}
\end{equation}
where $H=H^\trans\succ0$ and $(\xi_k)$ are i.i.d.  For the mean and covariance statements we assume $\E\xi_k=0$ and $\E\abs{\xi_k}^2<\infty$ with covariance $\Omega$; the causal-stationarity statement below only needs a logarithmic moment.  The subsequent stable-domain-of-attraction subsection replaces the finite-variance and centering assumptions explicitly.  Let $e_k=\Theta_k-\theta^\star$.  Then
\begin{equation}
 e_{k+1}=(I_d-\gamma H)\{(1+\beta)e_k-\beta e_{k-1}\}
 +\gamma\xi_{k+1}.
 \label{eq:quadratic-recursion}
\end{equation}

\begin{theorem}[Exact quadratic stability]
\label{thm:quadratic-stability}
Let $0\leq\beta<1$.  The deterministic part of \eqref{eq:quadratic-recursion} is Schur stable if and only if
\begin{equation}
 0<\gamma<
 \frac{2(1+\beta)}{(1+2\beta)\lambda_{\max}(H)}.
 \label{eq:exact-region}
\end{equation}
Under \eqref{eq:exact-region} and $\E\log(1+|\xi_0|)<\infty$, \eqref{eq:quadratic-recursion} has a unique causal strictly stationary solution.  If $\E\abs{\xi_0}^p<\infty$ for some $p>0$, then the stationary solution has a finite $p$th moment and the recursion is pathwise geometrically contracting in an equivalent norm.
\end{theorem}

\begin{proof}
Diagonalize $H=U\diag(h_1,\ldots,h_d)U^\trans$.  Each scalar eigenmode satisfies
\begin{equation}
 e_{k+1}^{(h)}=(1+\beta)r_he_k^{(h)}-\beta r_he_{k-1}^{(h)}+\gamma\xi_{k+1}^{(h)},
 \qquad r_h=1-\gamma h.
 \label{eq:eigenmode}
\end{equation}
The characteristic polynomial is
\begin{equation}
 \lambda^2-(1+\beta)r_h\lambda+\beta r_h.
 \label{eq:charpoly}
\end{equation}
For a real quadratic polynomial $\lambda^2-a\lambda+b$, the Jury conditions are
\[
 \abs b<1,
 \qquad 1-a+b>0,
 \qquad 1+a+b>0.
\]
Here they become
\begin{align*}
 \abs{\beta r_h}&<1,\\
 1-r_h&=\gamma h>0,\\
 1+(1+2\beta)r_h&>0.
\end{align*}
The last inequality is $\gamma h<2(1+\beta)/(1+2\beta)$.  On this interval, $r_h\in(-1/(1+2\beta),1)$, so $\abs{\beta r_h}<1$ is automatic.  Requiring the condition for every eigenvalue gives \eqref{eq:exact-region}.  Schur stability gives geometrically decaying moving-average coefficients.  To justify convergence under the logarithmic moment, for each $a>0$,
$\sum_{j\geq1}\Pp\{\log(1+|\xi_{-j}|)>aj\}<\infty$ by the integral tail test.  Borel--Cantelli makes the innovations subexponential along the past, so the exponentially weighted series converges absolutely almost surely.  A finite $p$th moment gives $L^p$ convergence by Minkowski for $p\geq1$ and subadditivity for $0<p<1$.  Finally, the usual discrete Lyapunov series for the deterministic companion matrix gives an equivalent contraction norm.  These arguments prove the stationary, moment, and contraction conclusions.
\end{proof}

\begin{corollary}[Exact scalar stationary variance]
\label{cor:variance}
In dimension one, let $H=h>0$ and $\Var(\xi_k)=\sigma^2$.  Under \eqref{eq:exact-region},
\begin{equation}
 \Var(\Theta_0^\circ-\theta^\star)
 =\frac{\gamma\sigma^2\{1+\beta(1-\gamma h)\}}
 {h\{1-\beta(1-\gamma h)\}
 \{1+(1+2\beta)(1-\gamma h)\}}.
 \label{eq:variance-formula}
\end{equation}
\end{corollary}

\begin{proof}
Write $r=1-\gamma h$, $a=(1+\beta)r$, $c=-\beta r$, and $q=\gamma^2\sigma^2$.  The scalar recursion is $e_{k+1}=ae_k+ce_{k-1}+\gamma\xi_{k+1}$.  If $v=\E e_k^2$ and $r_1=\E(e_ke_{k-1})$, the Yule--Walker equations give
\[
 r_1=\frac{a}{1-c}v
\]
and
\[
 v=a^2v+c^2v+2acr_1+q.
\]
Solving,
\begin{equation}
 v=\frac{q(1-c)}{(1+c)\{(1-c)^2-a^2\}}.
 \label{eq:yulewalker}
\end{equation}
Substituting $a,c,q$ and using $1-r=\gamma h$ reduces \eqref{eq:yulewalker} to \eqref{eq:variance-formula}.
\end{proof}

\begin{remark}[Small-step and high-momentum limits do not commute]
\label{rem:noncommuting}
For every fixed $\beta<1$, expansion of \eqref{eq:variance-formula} gives
\[
 \Var(\Theta_0^\circ-\theta^\star)
 \sim \frac{\gamma\sigma^2}{2h(1-\beta)},\qquad\gamma\downarrow0.
\]
This confirms the $\gamma/(1-\beta)$ dependence of Theorem~\ref{thm:scale} in its small-step regime.  It does not imply divergence as $\beta\uparrow1$ with $\gamma$ fixed.  For fixed $0<\gamma h<4/3$, the exact variance instead tends to
\[
 \frac{\sigma^2(2-\gamma h)}{h^2(4-3\gamma h)}.
\]
The latter is positive and finite, and tends to $\sigma^2/(2h^2)$ when $\gamma\downarrow0$ afterwards.  Reversing the two limits gives zero.  Thus the small-step expansion is nonuniform near unit momentum; neither the explicit stationary bound nor the fixed-$\beta$ expansion is asserted uniformly there.
\end{remark}

The stationary process has the matrix moving-average representation
\begin{equation}
 e_k=\sum_{j=0}^{\infty}B_j\xi_{k-j},
 \label{eq:ma-rep}
\end{equation}
with transfer function
\begin{equation}
 \mathcal B(z)
 =\gamma\left[I_d-(1+\beta)(I_d-\gamma H)z
 +\beta(I_d-\gamma H)z^2\right]^{-1}.
 \label{eq:transfer}
\end{equation}

\begin{theorem}[Momentum-invariant long-run covariance]
\label{thm:lrcov-quadratic}
Under \eqref{eq:exact-region} and $\E\abs{\xi_0}^2<\infty$, the long-run covariance of the stationary quadratic Nesterov iterates is
\begin{equation}
 \Sigma_{\gamma,\beta}
 =H^{-1}\Omega H^{-1}.
 \label{eq:lrcov-exact}
\end{equation}
In particular, it is independent of the step size and momentum parameter throughout the stability region.
\end{theorem}

\begin{proof}
For a stable linear process, the long-run covariance is the zero-frequency spectral matrix
\[
 \Sigma_{\gamma,\beta}=\mathcal B(1)\Omega\mathcal B(1)^\trans.
\]
At $z=1$,
\begin{align*}
 I_d-(1+\beta)(I_d-\gamma H)+\beta(I_d-\gamma H)
 &=I_d-(I_d-\gamma H)\\
 &=\gamma H.
\end{align*}
Hence $\mathcal B(1)=\gamma(\gamma H)^{-1}=H^{-1}$, proving \eqref{eq:lrcov-exact}.
\end{proof}

\begin{remark}[Bias--variance interpretation]
Momentum changes the transient roots and the one-time stationary variance \eqref{eq:variance-formula}, but it does not change the zero-frequency variance of the averaged quadratic process.  Thus fixed-parameter Nesterov acceleration can trade faster forgetting against larger short-run fluctuations without improving the first-order asymptotic covariance of the full iterate average.  This complements sample-complexity comparisons such as \citet{GaneshEtAl2023}.
\end{remark}

\subsection{A stable limit under infinite-variance innovations}

Assume now that $(\xi_k)$ are i.i.d. and belong to the multivariate domain of attraction of an $\alpha$-stable vector $S_\alpha$, $0<\alpha<2$: for some $a_n\uparrow\infty$ and centering $b_n$ when needed,
\begin{equation}
 a_n^{-1}\left(\sum_{k=1}^n\xi_k-b_n\right)
 \Rightarrow S_\alpha.
 \label{eq:innovation-stable}
\end{equation}

\begin{theorem}[Exact quadratic stable limit]
\label{thm:stable}
Under \eqref{eq:exact-region} and \eqref{eq:innovation-stable},
\begin{equation}
 a_n^{-1}\left\{
 \sum_{k=1}^n(\Theta_k^\circ-\theta^\star)-H^{-1}b_n
 \right\}
 \Rightarrow H^{-1}S_\alpha.
 \label{eq:stable-limit}
\end{equation}
The same limit holds from every deterministic initial pair.
\end{theorem}

\begin{proof}
Schur stability makes the coefficients in \eqref{eq:ma-rep} geometrically summable and
\[
 \sum_{j=0}^{\infty}B_j=\mathcal B(1)=H^{-1}.
\]
Let $C_j=\sum_{\ell>j}B_\ell$ and define the stationary series $v_k=\sum_{j\geq0}C_j\xi_{k-j}$, which converges because $(C_j)$ decays geometrically and regularly varying random vectors have a finite logarithmic moment.  The Beveridge--Nelson identity is
\begin{equation}
 e_k=H^{-1}\xi_k+v_{k-1}-v_k.
 \label{eq:BN}
\end{equation}
Summing gives
\[
 \sum_{k=1}^ne_k=H^{-1}\sum_{k=1}^n\xi_k+v_0-v_n.
\]
The stationary endpoint difference is tight, while $a_n\to\infty$, so $(v_0-v_n)/a_n\to0$ in probability.  Equation \eqref{eq:innovation-stable} and Slutsky's theorem prove \eqref{eq:stable-limit}.  For deterministic initialization, the difference from a stationary orbit is the deterministic companion matrix to the $k$th power applied to the finite random initial discrepancy.  Its absolute sum is finite almost surely by Schur stability, even when the stationary initial state has no first moment.  Dividing this tight discrepancy by $a_n$ transfers the limit.
\end{proof}

\begin{remark}[What GMC alone does not imply]
Fractional-moment GMC controls memory but does not itself imply regular variation.  For a general nonlinear Nesterov recursion, a stable limit additionally requires joint regular variation, anti-clustering, and a negligible-small-jump condition for the stationary state; see \citet{BasrakSegers2009,BartkiewiczEtAl2011,BasrakKrizmanicSegers2012,MikoschWintenberger2014,Mirek2011}.  The quadratic theorem is useful because all these issues collapse to the innovation domain-of-attraction assumption through the exact linear filter.
\end{remark}

\subsection{Rigidity and an exact random-curvature benchmark}
\label{sec:random-curvature}

The scalar examples below separate three questions: whether a particular metric contracts, whether a particular admissible recursion is mean-square stable, and whether all recursions obeying the assumptions are stable.  These are not equivalent questions.

\begin{proposition}[Rigidity when $L_2=\mu$]
\label{prop:rigidity}
Under Assumptions~\ref{ass:strong}--\ref{ass:lp} at $p=2$ with $L_2=\mu$, for every deterministic pair $\theta,\vartheta$,
\[
 G(\theta,X)-G(\vartheta,X)=\mu(\theta-\vartheta)\quad\text{almost surely}.
\]
If the sample maps are continuous, the identity holds simultaneously for every pair on an event of probability one.  Thus $G(\theta,X)=\mu(\theta-\theta^\star)+G(\theta^\star,X)$ on that event.
\end{proposition}
\begin{proof}
Put $v=\theta-\vartheta$ and $H=G(\theta,X)-G(\vartheta,X)$.  For $v\ne0$,
\[
 \mu|v|^2\leq\langle v,\E H\rangle
 \leq |v||\E H|\leq |v|\|H\|_2\leq\mu|v|^2.
\]
Equality throughout implies $\E H=\mu v$ and
$\E|H-\E H|^2=0$.  A countable dense set and continuity give simultaneous validity.  The case $v=0$ is immediate.
\end{proof}

In particular, at $\mu=L_2$ the gap between \eqref{eq:allbeta-bound} and the exact quadratic boundary is entirely a loss in the certificate, not a consequence of random curvature.  At $\beta=0.9$ and $\mu=L_2=1$, these boundaries are approximately $0.01163$ and $1.35714$, a ratio of $116.7$; at $\beta=0.99$ the ratio is approximately $1.31\times10^4$.  ``Optimal'' in Proposition~\ref{prop:optimal-weight} means only optimal for its fixed slow coordinate and its displayed scalar inequalities.

For a genuinely random-curvature comparison, let
\begin{equation}
 G(\theta,X_k)=A_k(\theta-\theta^\star)-\xi_k,
 \qquad \E A_k=\mu>0,\quad\E A_k^2=L^2,
 \label{eq:random-gradient}
\end{equation}
where $(A_k,\xi_k)$ are i.i.d., $\E\xi_k=0$, and the anchor has a finite second moment.  The scalar position recursion is a finite-lag random-coefficient autoregression of the type in Example~7 of \citet{ChenWu2016}:
\[
 e_{k+1}=(1-\gamma A_{k+1})\{(1+\beta)e_k-\beta e_{k-1}\}
          +\gamma\xi_{k+1}.
\]
Proposition~\ref{prop:ar2} gives a sufficient condition
$(1+2\beta)\|1-\gamma A_0\|_p<1$ whenever the multiplier and anchor have the stated $p$th moments.  The exact second-moment calculation below instead retains the signed covariance terms; its boundary is not asserted to follow from that coefficient-sum condition.  Synchronous differences do not involve $\xi_k$.  In dimension one, set
\[
 t=1-\mu\gamma,\qquad v=1-2\mu\gamma+L^2\gamma^2.
\]
Their second-moment vector $m_k=(\E x_k^2,\E x_kx_{k-1},\E x_{k-1}^2)^\trans$ obeys $m_{k+1}=T_{\gamma,\beta}m_k$, where
\begin{equation}
 T_{\gamma,\beta}=\begin{pmatrix}
 (1+\beta)^2v&-2\beta(1+\beta)v&\beta^2v\\
 (1+\beta)t&-\beta t&0\\
 1&0&0
 \end{pmatrix}.
 \label{eq:second-moment-map}
\end{equation}

\begin{theorem}[Exact scalar multiplicative-curvature region]
\label{thm:random-boundary}
For $0<\beta<1$, define
\begin{align}
 c_0&=2\mu(1-\beta^2),\notag\\
 c_1&=2\beta(1+2\beta)\mu^2-(1+\beta)L^2,\notag\\
 c_2&=\beta(1+2\beta)\mu L^2,\label{eq:boundary-coefficients}\\
 \Gamma_{\rm sc}(\beta)&=
 \frac{c_1+\sqrt{c_1^2+4c_2c_0}}{2c_2}.
 \label{eq:scalar-boundary}
\end{align}
The scalar difference recursion is $L^2$ GMC if and only if
\begin{equation}
 \rho(T_{\gamma,\beta})<1
 \quad\Longleftrightarrow\quad
 0<\gamma<\Gamma_{\rm sc}(\beta).
 \label{eq:scalar-iff}
\end{equation}
For $\beta=0$, the boundary is $2\mu/L^2$.  When the boundary case $\beta=1$ is considered solely for this scalar model, a nonempty mean-square stable interval exists exactly when $L^2<3\mu^2$, and it is
\begin{equation}
 0<\gamma<\frac{2(3\mu^2-L^2)}{3\mu L^2}.
 \label{eq:beta-one}
\end{equation}
\end{theorem}
\begin{proof}
Independence of the current multiplier and the past gives \eqref{eq:second-moment-map}.  In covariance-matrix notation the same operator is $C\mapsto\E[A_{\gamma,\beta}(A_k)C A_{\gamma,\beta}(A_k)^\trans]$.  It preserves the cone of positive semidefinite matrices.  Its powers decay geometrically if its spectral radius is below one.  Conversely, GMC for every deterministic initial vector implies geometric decay on every rank-one covariance and hence on their linear span, the full space of symmetric matrices.  This proves the first equivalence.  Equivalently, when $\rho(T)<1$, the adjoint series
\[
 P=\sum_{j=0}^\infty (\mathcal T^*)^j(I_2)
\]
converges and satisfies $\E[A^\trans PA]=P-I_2$, giving a quadratic contraction norm.

We verify the spectral-radius boundary explicitly.  Direct expansion yields
\begin{equation}
 \det(I_3-T_{\gamma,\beta})
 =\gamma(c_0+c_1\gamma-c_2\gamma^2).
 \label{eq:random-det}
\end{equation}
At $\gamma=0$, the eigenvalues are $1,\beta,\beta^2$.  For $\beta<1$, the simple eigenvalue at one has derivative $-2\mu/(1-\beta)$, as follows by differentiating $f(\lambda,\gamma)=\det(\lambda I_3-T_{\gamma,\beta})$: at $(\lambda,\gamma)=(1,0)$,
$f_\gamma=2\mu(1-\beta^2)$ and $f_\lambda=(1-\beta)(1-\beta^2)$.
The other two eigenvalues remain strictly inside the unit disk for small $\gamma$, so $\rho(T)<1$ for sufficiently small positive $\gamma$.  A positive linear map on the positive-semidefinite cone has its spectral radius as a nonnegative real eigenvalue; a proof of this finite-dimensional fact is included in Lemma~\ref{lem:cone-perron}.  Therefore any first exit from $\rho(T)<1$ must satisfy $\det(I_3-T)=0$.  Since $c_0,c_2>0$, the polynomial $c_0+c_1\gamma-c_2\gamma^2$ has exactly one positive zero, namely \eqref{eq:scalar-boundary}.  There is no exit before it.  Beyond it, the determinant is negative, whereas $\det(I_3-T)>0$ for a real Schur-stable matrix: real eigenvalue factors are positive and nonreal factors occur in conjugate pairs.  This proves \eqref{eq:scalar-iff}.  At $\beta=0$, the only nonzero eigenvalue is $v$, giving the stated boundary.

At $\beta=1$,
\[
 \det(I_3-T_{\gamma,1})
 =\gamma^2\{2(3\mu^2-L^2)-3\mu L^2\gamma\}.
\]
If $L^2<3\mu^2$ and \eqref{eq:beta-one} holds, continuity from $\beta<1$ gives $\rho(T_{\gamma,1})\leq1$.  The positive determinant and the cone Perron property exclude equality.  Outside this interval the determinant is nonpositive, which precludes strict Schur stability.  This also proves that no positive interval exists when $L^2\geq3\mu^2$.
\end{proof}

\begin{corollary}[Three high-momentum regimes in the scalar model]
\label{cor:scalar-regimes}
As $\beta\uparrow1$,
\begin{equation}
 \Gamma_{\rm sc}(\beta)\sim
 \begin{cases}
 \displaystyle\frac{2(3\mu^2-L^2)}{3\mu L^2},&L^2<3\mu^2,\\[2mm]
 \displaystyle\frac{2}{3\mu}\sqrt{1-\beta},&L^2=3\mu^2,\\[2mm]
 \displaystyle\frac{2\mu}{L^2-3\mu^2}(1-\beta),&L^2>3\mu^2.
 \end{cases}
 \label{eq:scalar-three-regimes}
\end{equation}
\end{corollary}
\begin{proof}
In \eqref{eq:boundary-coefficients}, $c_0\sim4\mu(1-\beta)$, $c_2\to3\mu L^2$, and $c_1\to2(3\mu^2-L^2)$.  For a positive limiting $c_1$, the positive root tends to $c_1/c_2$; for a negative limiting $c_1$, it is asymptotic to $c_0/(-c_1)$.  At $L^2=3\mu^2$, $c_1=O(1-\beta)$ and the root is asymptotic to $\sqrt{c_0/c_2}$.  These are exactly the three displayed expressions.
\end{proof}

The instance \eqref{eq:random-gradient} saturates the two mean-only assumptions.  Thus its stability region is a \emph{necessary upper envelope} for any sufficient condition valid uniformly over that class.  It is not proved to be the worst case, and \eqref{eq:scalar-boundary} is not a sufficient nonlinear or multivariate theorem.  The exact transition is $L/\mu=\sqrt3$, rather than an approximate numerical threshold.  The nonzero interval at $\beta=1$ when $L/\mu<\sqrt3$ reflects Nesterov's look-ahead damping.  Heavy ball at $\beta=1$ has deterministic companion determinant one, and even its mean recursion cannot be Schur stable; mean-square contraction would imply mean contraction, so it is impossible there.

\section{Numerical illustrations and checked certificates}
\label{sec:numerics}

The experiments distinguish a sufficient condition, an exact model-specific boundary, a numerically verified matrix certificate, and an observed Monte Carlo rate.  These are different objects.  The accompanying script uses NumPy, SciPy, and Matplotlib, with root seed $20260923$, and saves the numerical values and full-precision matrices used below.  No experiment is presented as a benchmark for large-scale neural-network training.

\subsection{What is lost by the explicit metric?}
\label{sec:num-regions}

First set $\mu=L=h=1$.  Proposition~\ref{prop:rigidity} shows that this is deterministic scalar curvature plus additive noise, not random curvature.  The direct mean-square condition has boundary $\gamma=2(1-\beta)$ for $\beta<1/2$, whereas the explicit quadratic metric gives \eqref{eq:allbeta-bound}.  The exact Schur boundary is $2(1+\beta)/(1+2\beta)$.  A freely chosen single-mode Lyapunov matrix recovers every strictly stable point, by the discrete Lyapunov series.  Thus the gap in Figure~\ref{fig:regions} is entirely a certificate loss.  Optimizing a single weight with the slow coordinate fixed does not eliminate it.

\begin{figure}[htbp]
 \centering
 \includegraphics[width=0.88\textwidth]{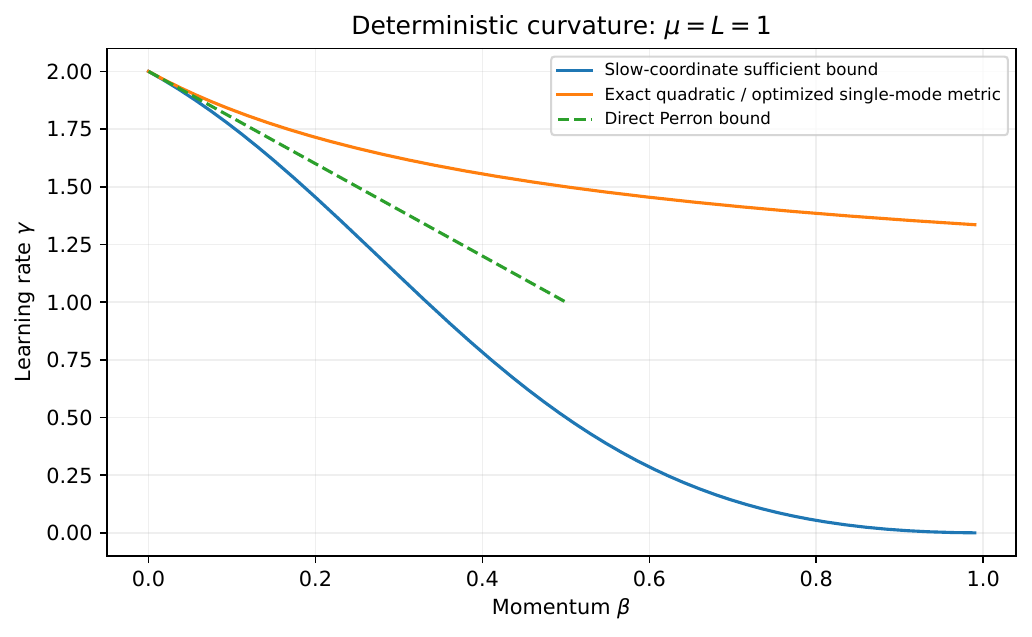}
 \caption{Deterministic curvature, $\mu=L=h=1$.  The exact quadratic curve is also the supremum obtainable by a freely chosen single-mode quadratic norm.  The much smaller slow-coordinate curve is not an intrinsic stability limit.}
 \label{fig:regions}
\end{figure}

For genuinely random curvature, compare the exact scalar benchmark of Theorem~\ref{thm:random-boundary} with the same sufficient bound, at $\mu=1$ and $L=1.5$ or $L=3$.  Table~\ref{tab:scalar-comparison} and Figures~\ref{fig:mean15}--\ref{fig:mean3} quantify the difference.  The scalar boundary is a necessary upper envelope for a universal mean-only theorem, not a sufficient region for the nonlinear class.  The distinction matters: the small-$L/\mu$ scalar example has a nonvanishing limiting region at unit momentum, while the fixed-coordinate certificate always vanishes quadratically.

\begin{table}[htbp]
\centering\small
\caption{Exact scalar random-curvature boundaries versus the explicit slow-coordinate sufficient bound, $\mu=1$. Ratios compare model-specific stability with a universal sufficient certificate.}
\label{tab:scalar-comparison}
\begin{tabular}{rrrrr}\toprule
$L$&$\beta$&$\Gamma_{\rm sc}$&$\Gamma_{\rm NAG}$&Ratio\\\midrule
1.5&0.50&0.566242&0.222222&2.5\\
1.5&0.80&0.404160&0.024024&16.8\\
1.5&0.90&0.334987&0.005168&64.8\\
1.5&0.95&0.289680&0.001198&241.8\\
1.5&0.99&0.239403&0.000045&5306.3\\
3&0.50&0.119297&0.055556&2.1\\
3&0.80&0.055083&0.006006&9.2\\
3&0.90&0.029835&0.001292&23.1\\
3&0.95&0.015689&0.000299&52.4\\
3&0.99&0.003290&0.000011&291.7\\
\bottomrule\end{tabular}\end{table}

\begin{figure}[htbp]
 \centering
 \includegraphics[width=0.86\textwidth]{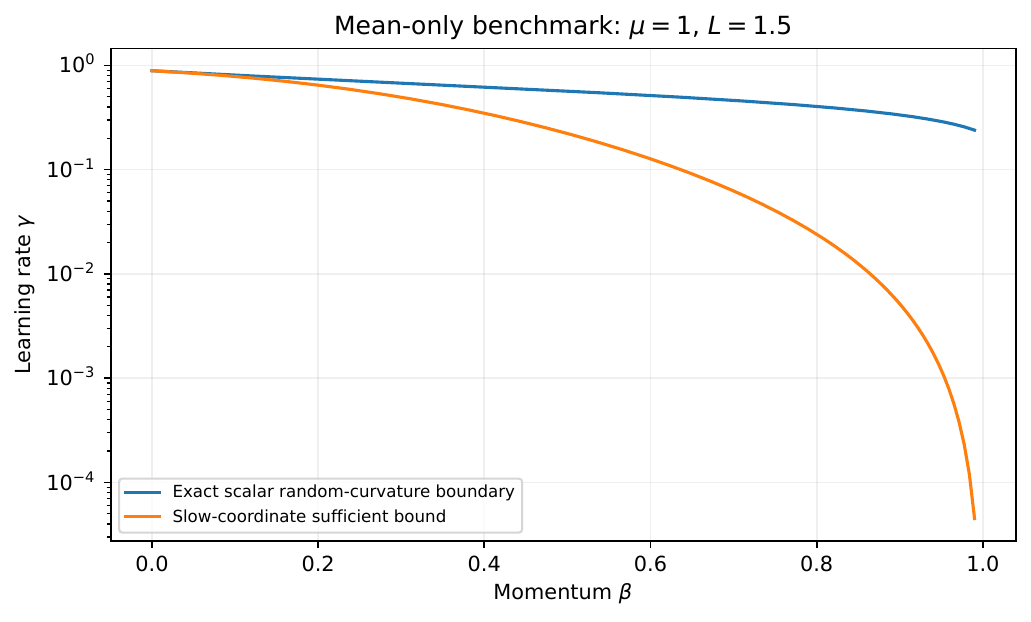}
 \caption{Mean-only comparison for $\mu=1,L=1.5$.  The scalar exact boundary tends to $2/9$ as $\beta\uparrow1$, whereas the explicit slow-coordinate bound tends to zero.}
 \label{fig:mean15}
\end{figure}
\begin{figure}[htbp]
 \centering
 \includegraphics[width=0.86\textwidth]{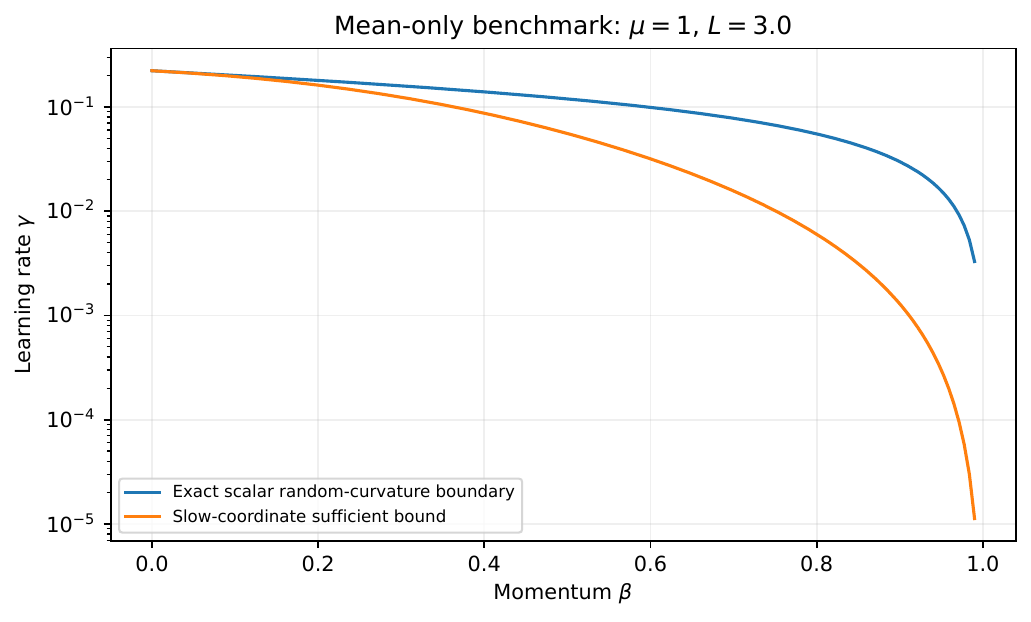}
 \caption{Mean-only comparison for $\mu=1,L=3$.  The exact scalar boundary decays linearly in $1-\beta$, not quadratically.  No sufficiency claim for general nonlinear models is made for the upper curve.}
 \label{fig:mean3}
\end{figure}
\FloatBarrier

\subsection{Endpoint and mean-only quadratic certificates}
\label{sec:num-lmi}

For the samplewise sectors $[1,1.5]$ and $[1,3]$, we optimize a trace-one $2\times2$ matrix in the endpoint inequalities of Theorem~\ref{thm:lmi}.  The code maximizes a common eigenvalue slack using SLSQP, then independently recomputes all eigenvalues of the candidate matrix and residuals.  Only positive-slack candidates are reported as certificates.  This is a small semidefinite feasibility problem mathematically; the numerical implementation does not require a dedicated semidefinite solver.  Failure of the search is not a proof of infeasibility, and a reported upper edge is not proved globally optimal.

The search locates a strictly verified step approximately $0.923076$ for $[1,1.5]$, $\beta=0.8$, compared with the per-mode Schur upper bound $0.923077$ and the explicit bound $0.024024$.  For $[1,3]$, $\beta=0.9$, it locates approximately $0.434862$, compared with $0.452381$ and $0.001292$.  These improvements are factors of approximately $38$ and $337$ over the explicit sufficient steps.  The per-mode upper bound is necessary for a common quadratic certificate; it is not by itself sufficient under time-varying curvature.

\begin{figure}[htbp]
 \centering
 \includegraphics[width=0.86\textwidth]{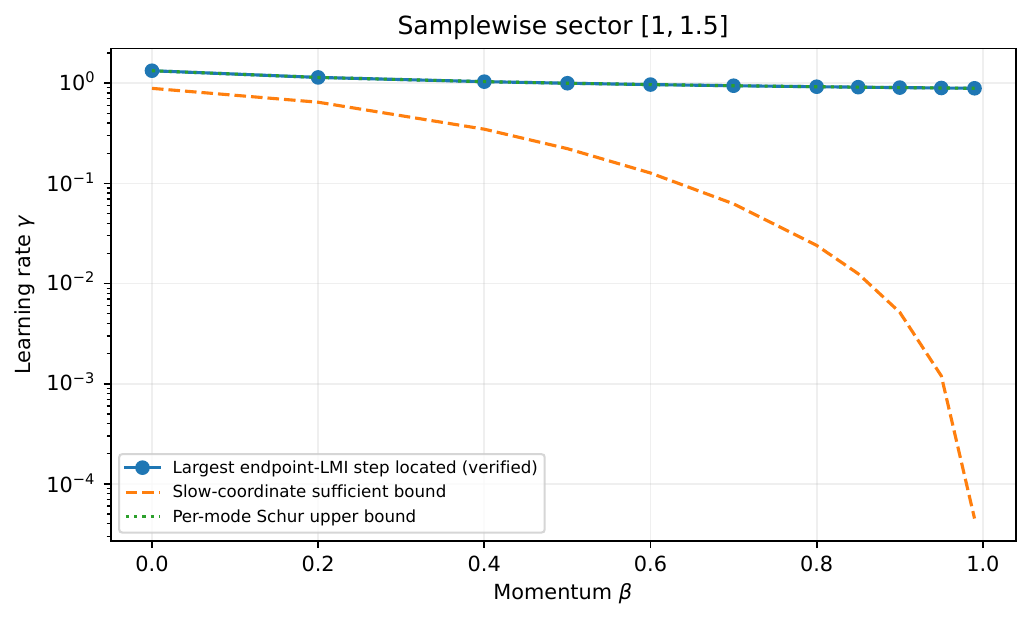}
 \caption{Samplewise sector $[1,1.5]$.  Dots are the largest strictly verified steps located by the specified grid and local refinement; connecting lines are visual guides, not certification of every intervening parameter value.  The upper curve is the per-mode Schur bound.}
 \label{fig:sector15}
\end{figure}
\begin{figure}[htbp]
 \centering
 \includegraphics[width=0.86\textwidth]{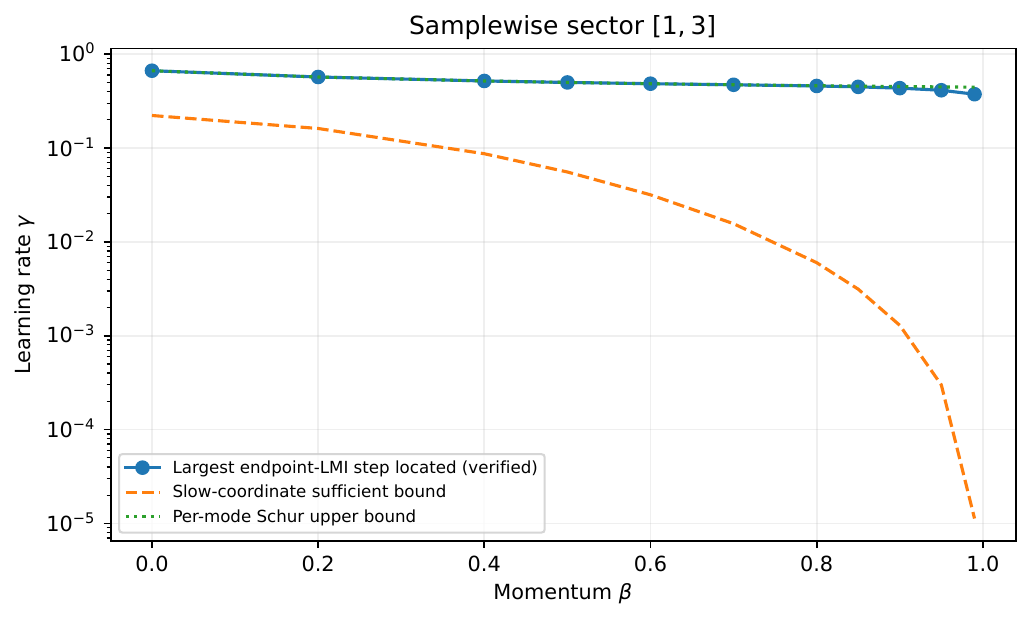}
 \caption{Samplewise sector $[1,3]$.  The optimized common metric substantially improves the fixed slow-coordinate metric.  These dots do not assert that every smaller step is feasible for a common quadratic norm; no monotonicity of feasible steps is assumed.}
 \label{fig:sector3}
\end{figure}

For reproducibility, the following rounded matrices satisfy the endpoint inequalities with the displayed squared factor $\varrho$, in the two-iterate coordinates $(x_k,x_{k-1})$.  Full precision is in the numerical output:
\begin{align}
 (\beta,\gamma,[\mu_s,L_s],\varrho)&=(0.8,0.9,[1,1.5],0.87),\notag\\*
 P_0&=\begin{pmatrix}0.90771236&-0.25957077\\-0.25957077&0.09228764\end{pmatrix},
 \label{eq:num-P1}\\
 (\beta,\gamma,[\mu_s,L_s],\varrho)&=(0.9,0.42,[1,3],0.97),\notag\\*
 P_0&=\begin{pmatrix}0.75669343&-0.25348634\\-0.25348634&0.24330657\end{pmatrix}.
 \label{eq:num-P2}
\end{align}
The minimum of the eigenvalues of $P_0$ and both endpoint residuals $\varrho P_0-A(h)^\trans P_0A(h)$ exceeds $0.0026$ and $0.0040$, respectively; the rounding does not consume these margins.  The accompanying symbolic-check script treats these printed decimals as exact rational numbers and verifies positive leading principal minors of $P_0$ and both residuals.  Thus these two displayed certificates do not rely on a solver status or on floating-point eigenvalue signs.

The mean-only $S$-procedure is different because it does not assume a positive samplewise sector.  At $(\mu,L,\beta,\gamma)=(1,1.5,0.8,0.03)$, a trace-one certificate at $\varrho=1$ is
\[
 P=\begin{pmatrix}0.09779322&0.11092930\\0.11092930&0.90220678\end{pmatrix},
 \qquad \lambda_1=0.03246927,\quad\lambda_2=0.00787110.
\]
The positive eigenvalue slack at $\varrho=1$ exceeds $0.0018$, including after rounding.  More strongly, the symbolic-check script verifies positive leading principal minors of the negative $3\times3$ residual using the printed rational decimals at $\varrho=0.999$.  Theorem~\ref{thm:mean-lmi} therefore gives the explicit squared factor $0.999$; this step is larger than $0.024024$.  This is an exhibited improvement, not a claim that the general mean-only LMI reaches the scalar upper envelope or attains optimal high-momentum scaling.

An additional caution emerged from the parameter search.  For $\beta=0.99$ and sector $[1,3]$, positive certificates are found near both $\gamma=10^{-4}$ and $\gamma=0.3$, while the implementation does not certify several intermediate steps.  Accordingly the search scans the step range before local refinement rather than assuming monotone feasibility in $\gamma$.  The negative searches do not prove a gap and do not justify interpreting the plotted largest-found steps as certificates of all smaller steps.
\FloatBarrier

\subsection{Certified and observed forgetting in a nonlinear sector model}
\label{sec:num-forgetting}

Consider
\begin{equation}
 G(\theta,\xi)=\theta+\tfrac12\tanh\theta+\xi,
 \label{eq:nonlinear-example}
\end{equation}
with i.i.d. Student-$t_3$ innovations.  The sample Jacobian is deterministic and lies in $[1,1.5]$.  Thus this experiment exercises the samplewise-sector theory, not the extra scope of the mean-only theorem.  We use $\beta=0.8$ and compare the initial pair $(10,10)$ with $(0,0)$ using shared innovations in each of $1800$ independent replications.

At $\gamma=0.015$, the explicit slow-coordinate energy has certified squared factor $0.9979116$.  An endpoint certificate improves this to $0.8252862$ in its own equivalent norm.  A log-linear fit of the observed Euclidean squared separation over iterations $50$--$200$ gives $0.78156$ per step.  Figure~\ref{fig:forgetting} compares each normalized metric with its own bound; it does not compare an unadjusted Euclidean distance with a bound stated in another norm.  The explicit rate is conservative even well inside its stable region.

\begin{figure}[htbp]
 \centering
 \includegraphics[width=0.9\textwidth]{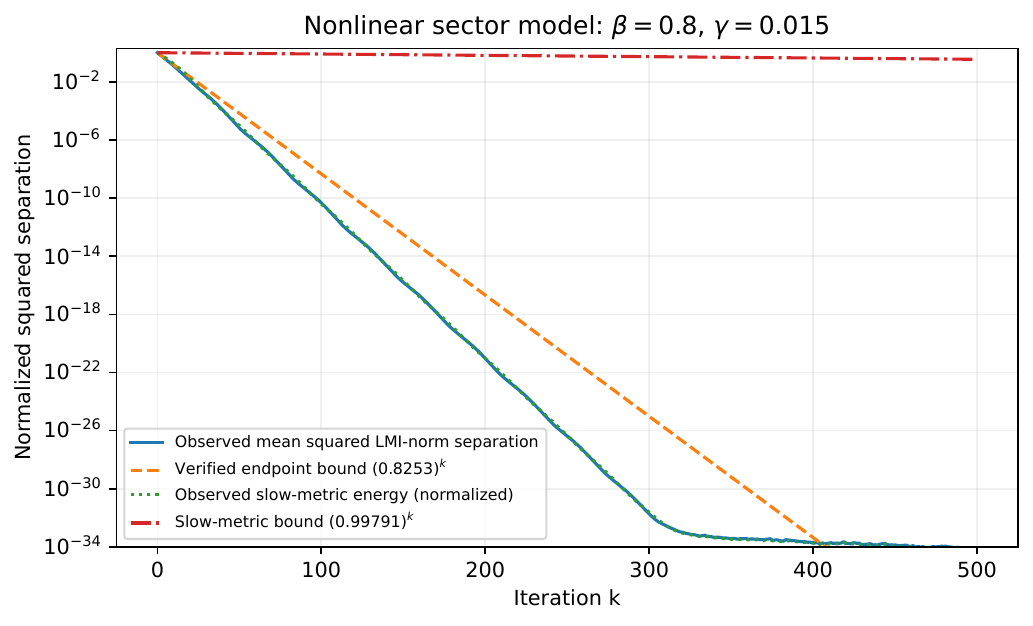}
 \caption{Nonlinear model, $\beta=0.8$, $\gamma=0.015$.  Each observed squared metric is normalized by its initial value and compared with the corresponding certified factor.  Late flattening near machine precision is floating-point round-off, not a nonzero mathematical coupling error.}
 \label{fig:forgetting}
\end{figure}

The same model is also run at $\gamma=0.9$, justified by the verified endpoint matrix \eqref{eq:num-P1}, far beyond the slow-coordinate interval.  Its computed worst endpoint squared factor is $0.8665990<0.87$.  Figure~\ref{fig:forgetting-large} shows synchronous forgetting at this larger step.  Trajectories eventually coalesce in floating-point arithmetic; no observed asymptotic slope is fitted after that event.  The experiment demonstrates applicability of the improved certificate, not optimality of $\gamma=0.9$ for inference or optimization.

\begin{figure}[htbp]
 \centering
 \includegraphics[width=0.9\textwidth]{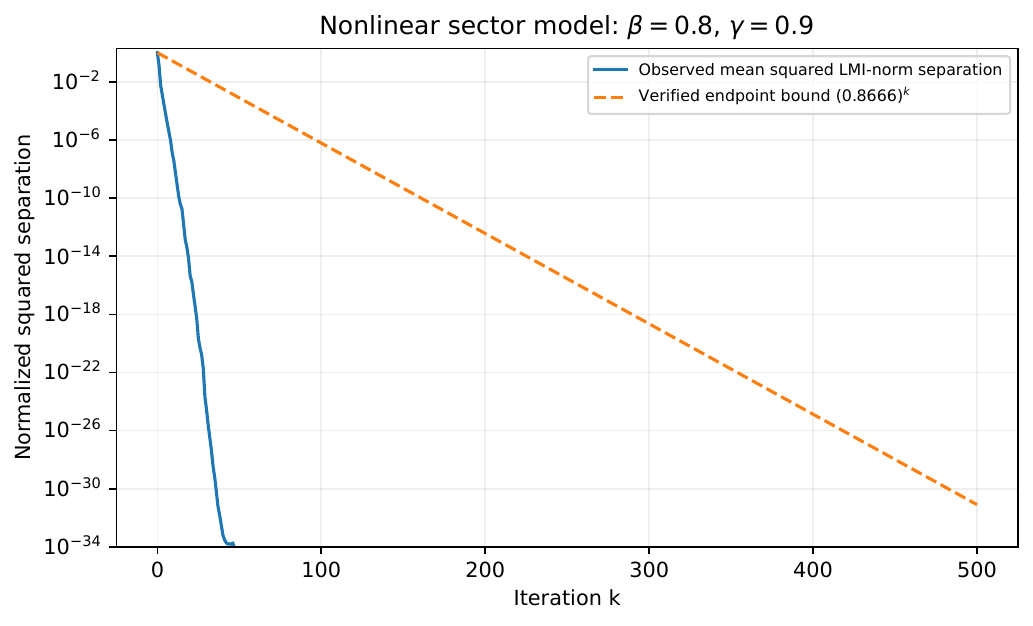}
 \caption{The same nonlinear sector model at $\gamma=0.9$, certified by the endpoint LMI rather than the explicit slow-coordinate bound.  Values below the plotting range or equal to zero after floating-point coalescence are not used to estimate a decay rate.}
 \label{fig:forgetting-large}
\end{figure}
\FloatBarrier

\subsection{A genuinely mean-only random-curvature experiment}
\label{sec:num-multiplicative}

Let $A_k=1+\sqrt{1.25}\,\varepsilon_k$, where $\varepsilon_k$ is symmetric on $\{-1,1\}$.  Then $\E A_k=1$, $\E A_k^2=2.25$, and the two sample slopes are approximately $-0.1180$ and $2.1180$.  Thus the mean-only assumptions hold with $(\mu,L_2)=(1,1.5)$, but no positive samplewise curvature lower bound holds.  Set $\beta=0.8$, $\gamma=0.015$ and start the synchronous two-iterate difference at $(1,1)$.  Additive innovations cancel in this difference recursion and need not be simulated.

We compare $60000$ Monte Carlo difference trajectories with the exact second-moment map \eqref{eq:second-moment-map}.  Its spectral radius is $0.8258465$, compared with the explicit theorem's squared factor $0.9979116$.  Figure~\ref{fig:multiplicative} shows both the substantial rate conservatism and the agreement of the finite-sample moment estimate with the exact recursion over most of the range.  Very small late moments are sensitive to rare multiplicative paths, so empirical agreement is not asserted at arbitrary relative precision.

\begin{figure}[htbp]
 \centering
 \includegraphics[width=0.9\textwidth]{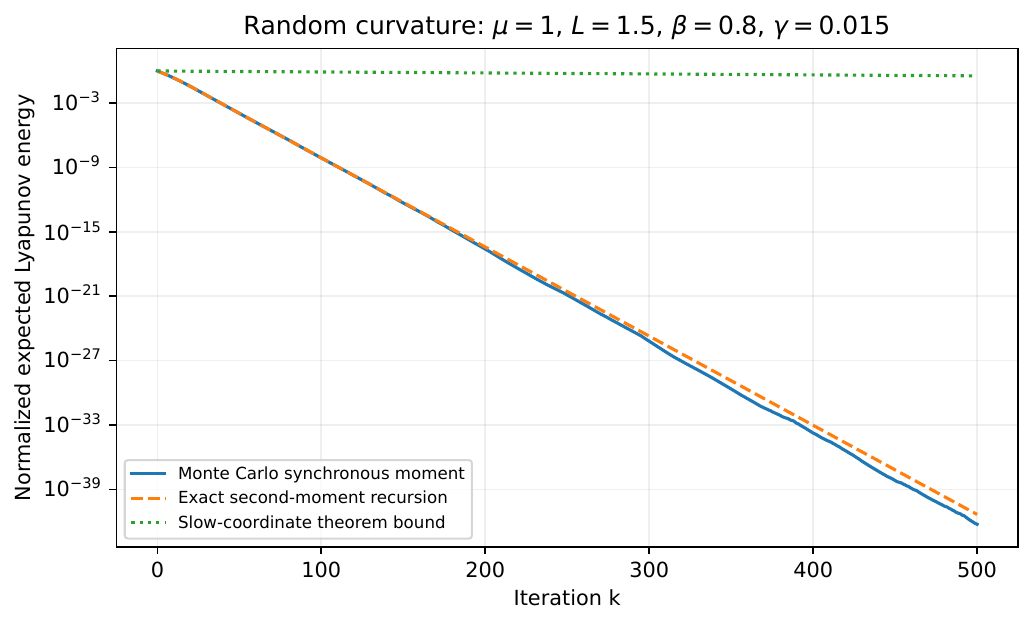}
 \caption{Mean-only example with a negative sample slope of positive probability.  The exact second-moment recursion is available independently of Monte Carlo sampling and gives a model-specific benchmark for the theorem's bound.}
 \label{fig:multiplicative}
\end{figure}
\FloatBarrier

\subsection{Fractional moments with infinite-variance innovations}
\label{sec:num-fractional}

Use \eqref{eq:nonlinear-example} with symmetric Pareto innovations satisfying $\Pp(|\xi|>x)=x^{-1.5}$ for $x\geq1$.  They have mean zero and a finite $1.4$th moment, but infinite variance.  We choose $p=1.4$, $\beta=0.1$, $\gamma=0.003$ and use $6000$ independent shared-noise pairs initialized at $(10,10)$ and $(0,0)$.  The direct Perron calculation gives
\[
 r_{\gamma,\beta,p}=0.9980599484,
 \qquad \eta_{\gamma,\beta,p}=0.0111351141.
\]
The empirical product distance
$\|y_k\|_{1.4}+\eta\|s_k\|_{1.4}$ and its theoretical upper bound are shown in Figure~\ref{fig:fractional}.  This experiment checks the direct fractional-moment theorem at a permitted small momentum.  It does not numerically validate the tiny large-momentum interval of Theorem~\ref{thm:allp}; that theorem is established by its proof and algebraic checks, not by extrapolating this experiment.

\begin{figure}[htbp]
 \centering
 \includegraphics[width=0.9\textwidth]{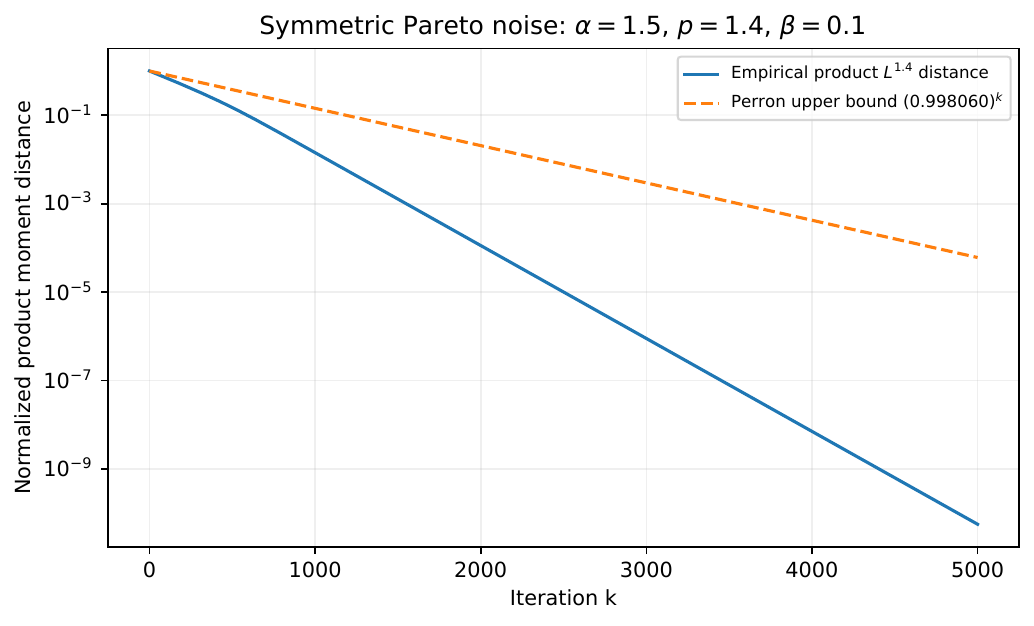}
 \caption{Fractional-moment synchronous coupling under infinite-variance Pareto noise.  The finite $p$th moment is $p=1.4<1.5$; no variance estimate is used in the displayed product distance.}
 \label{fig:fractional}
\end{figure}
\FloatBarrier

\subsection{The full-orbit central limit theorem from separated initial pairs}
\label{sec:num-clt}

For the scalar additive quadratic model take $h=1$, $\gamma=0.4$, $\beta=0.9$, and independent Rademacher innovations $\xi_k\in\{-1,1\}$.  Its long-run variance is exactly one.  Across $3000$ independent streams, compute $\sqrt n\,\bar\Theta_n$ with $n=40000$ for initial pairs $(-30,-30)$, $(0,0)$ and $(30,30)$, retaining \emph{every} iterate in the average.  The three starts share a stream within a replication; different replications are independent.  For efficiency the code adds the exact homogeneous transient to the zero-start orbit, which is algebraically identical to running three separate linear recursions.

\begin{table}[htbp]\centering\small
\caption{Full-orbit quadratic CLT experiment: $n=40000$, $3000$ replications, and known long-run variance one. Coverage uses $\pm z_{0.975}$ on the standardized statistic.}
\label{tab:clt}
\begin{tabular}{rrrrr}\toprule
Initial pair&Mean&Standard deviation&Exact initial shift&Coverage\\\midrule
$(-30,-30)$&-0.0165&1.0080&-0.0225&0.9463\\
$(0,0)$&0.0060&1.0080&0.0000&0.9483\\
$(30,30)$&0.0285&1.0080&0.0225&0.9477\\
\bottomrule\end{tabular}\end{table}

Figure~\ref{fig:clt} and Table~\ref{tab:clt} show the normal approximation and the small remaining deterministic shifts, whose exact standardized values are $-0.0225$, $0$, and $0.0225$.  This illustrates initialization robustness without deleting burn-in.  The coverage calculation uses the known asymptotic long-run variance, not an estimated standard error; it is not a validation of a new confidence-interval estimator.

\begin{figure}[htbp]
 \centering
 \includegraphics[width=0.79\textwidth]{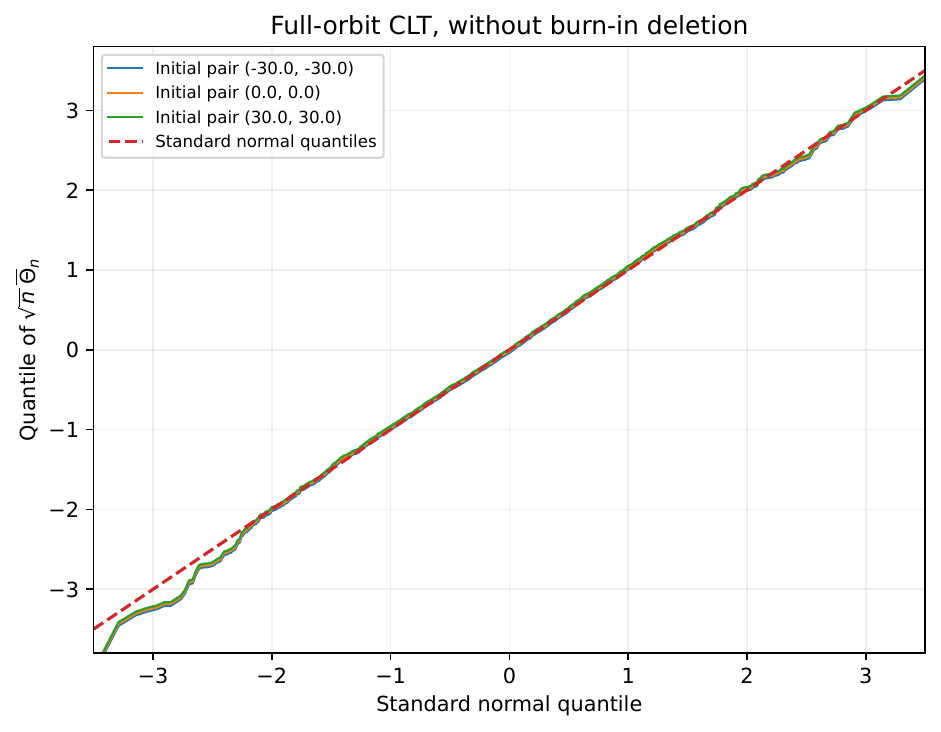}
 \caption{Normal quantile comparison for full-orbit averages from three separated initial pairs.  The innovations are non-Gaussian; the normal comparison concerns the iterate average.  No initial iterates are discarded.}
 \label{fig:clt}
\end{figure}
\FloatBarrier

\subsection{Marginal variance versus long-run variance}
\label{sec:num-variance}

Take $h=\sigma=1$, $\gamma=0.4$ and Gaussian additive innovations.  For each of six momentum values, use $12$ independent runs of $400000$ retained observations following $5000$ burn-in steps.  The marginal variance is estimated by the usual sample variance.  The long-run variance is estimated by $1000$ times the sample variance of $400$ nonoverlapping batch means; standard errors across the $12$ independent runs describe Monte Carlo variation.  Burn-in deletion in this stationary-moment experiment is separate from the full-orbit CLT experiment above.

\begin{table}[htbp]\centering\small
\caption{Quadratic marginal and long-run variance estimates. Parentheses contain Monte Carlo standard errors across $12$ independent runs; the exact long-run variance equals one.}
\label{tab:variance-lrv}
\begin{tabular}{rrrr}\toprule
$\beta$&Exact marginal variance&Estimated marginal variance&Estimated long-run variance\\\midrule
0.00&0.25000&0.24997 (0.00029)&0.9931 (0.0163)\\
0.30&0.29368&0.29371 (0.00028)&1.0110 (0.0220)\\
0.50&0.33766&0.33746 (0.00046)&1.0196 (0.0199)\\
0.70&0.40136&0.40157 (0.00033)&1.0003 (0.0141)\\
0.90&0.49968&0.50043 (0.00066)&1.0475 (0.0172)\\
0.95&0.53302&0.53283 (0.00047)&0.9540 (0.0187)\\
\bottomrule\end{tabular}\end{table}

Figure~\ref{fig:variance} reproduces the increase in marginal variance, while Figure~\ref{fig:lrv} displays the very different behavior of the long-run variance, whose population value is one for every displayed momentum.  Empirical long-run estimates deviate by a few percent owing to Monte Carlo variation and finite-batch effects; the error bars quantify sampling variability across runs, not all sources of estimator bias.  We do not require every bar to cover the exact value.  The contrast illustrates the zero-frequency statement of Theorem~\ref{thm:lrcov-quadratic}, rather than claiming that momentum leaves the full spectrum unchanged.  Recursive alternatives for long-run variance estimation are developed by \citet{Wu2009,ZhuChenWu2023}.

\begin{figure}[htbp]
 \centering
 \includegraphics[width=0.80\textwidth]{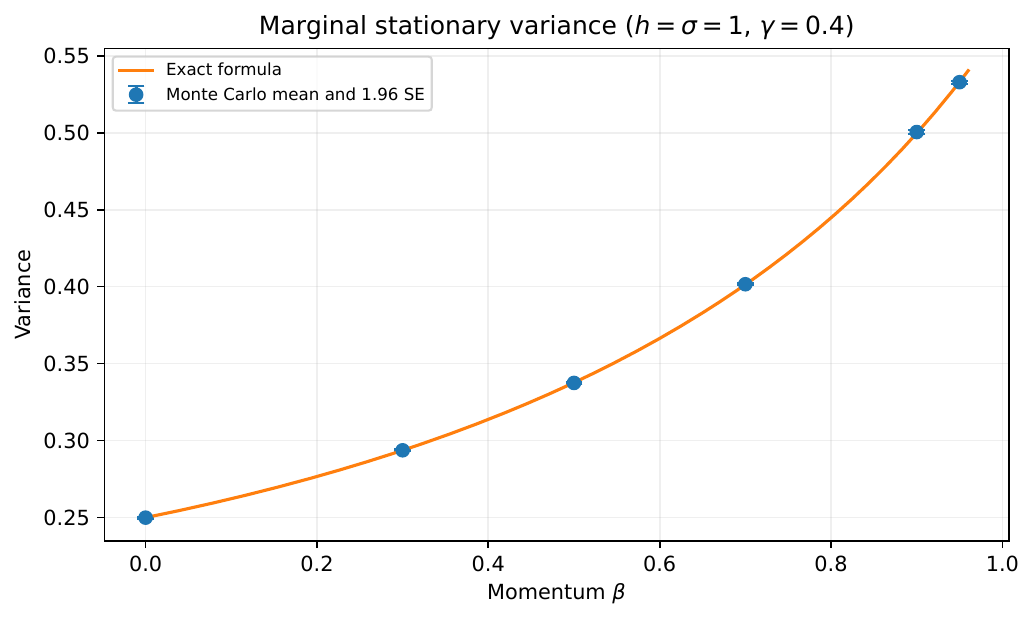}
 \caption{Marginal stationary variance for the additive quadratic model: exact formula and Monte Carlo means with $1.96$ standard errors across independent runs.}
 \label{fig:variance}
\end{figure}
\begin{figure}[htbp]
 \centering
 \includegraphics[width=0.80\textwidth]{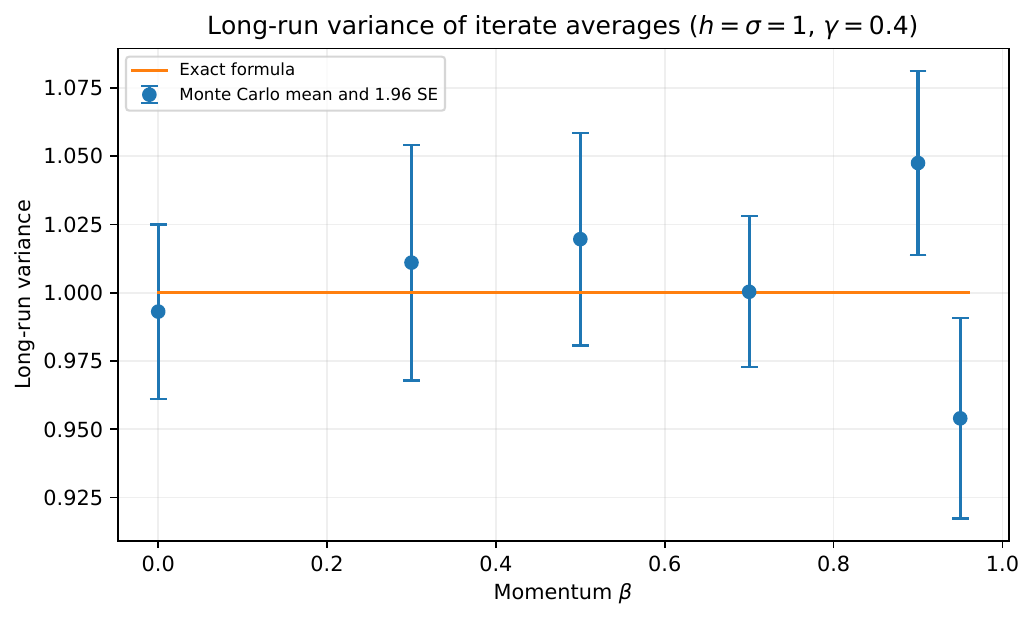}
 \caption{Long-run variance for the same parameters, estimated by nonoverlapping batch means.  The exact value is one for every momentum, although the marginal variance in Figure~\ref{fig:variance} changes substantially.}
 \label{fig:lrv}
\end{figure}
\FloatBarrier

\section{Discussion and extensions}
\label{sec:discussion}

The theory separates moment assumptions, momentum parameters, and the information available about curvature.  The direct Perron theorem has a restrictive momentum range, but provides a simple finite-$p$ criterion.  The explicit quadratic theorem covers every fixed momentum in mean square, with a conservative step interval.  The power-Lyapunov theorem supplies finite-$p$ contraction for arbitrary fixed momentum without adding a second moment when $p<2$, but its quantitative interval can be very small.  General quadratic certificates improve some of these ranges; under samplewise sectors, verified endpoint matrices can make the difference particularly large.  No sufficient certificate is identified with the true nonlinear stability boundary.

The exact scalar random-curvature calculation helps distinguish certificate loss from genuine instability.  Its phase transition at $L_2/\mu=\sqrt3$ and its constant, square-root, and linear high-momentum scales are properties of that admissible scalar model.  They are necessary benchmarks for a universal theorem, not a solution of the worst-case nonlinear or multivariate stability problem.  In particular, the favorable scalar behavior at $\beta=1$ does not extend the general GMC theorems to unit momentum.

The statistical results use the following established mechanism, developed for nonlinear autoregressions by \citet{ChenWu2016}, after the algorithm-specific contraction has been verified:
\[
 \text{GMC}\ \Longrightarrow\ \text{causal stationarity and physical dependence}
 \ \Longrightarrow\ \text{limit theory for the observed orbit}.
\]
The central limit theorem requires a second-moment contraction.  The sharper Gaussian coupling requires the stated higher moment and its own contraction condition.  These conclusions center at the invariant mean, which need not equal the optimizer for a fixed nonzero step.  They concern initialization robustness, not a conditional limit theorem given an arbitrary random environment.

Several questions remain quantitatively important.  Useful all-momentum regions under only a fractional moment would improve on the existence interval of Theorem~\ref{thm:allp}.  Optimizing a general mean-only metric, richer quadratic constraints, or multistep contraction may improve the high-momentum scaling.  A look-ahead coefficient $\nu\in[0,1]$ interpolates between heavy ball and Nesterov through
\[
 Y_k=\Theta_k+\nu\beta(\Theta_k-\Theta_{k-1}),\qquad
 \Theta_{k+1}=\Theta_k+\beta(\Theta_k-\Theta_{k-1})-\gamma G(Y_k,X_{k+1});
\]
its optimal certificates need not interpolate linearly.  Parameter-derivative processes and invariant-bias expansions could support Richardson--Romberg corrections, complementing stochastic modified-equation and continuous-time analyses \citep{LiTaiE2019,SirignanoSpiliopoulos2020,LiuChenZhouZhao2021,JinEtAl2025}.  Online covariance estimation is another natural application of the dependence bounds \citep{Wu2009,ZhuChenWu2023}.  Proximal or projected recursions would require checking nonexpansiveness in the \emph{augmented contraction metric}, which is not automatic from Euclidean nonexpansiveness.  Nonconvex extensions would need additional recurrence and communication assumptions before global uniqueness could be asserted; see related work on benign nonconvex acceleration \citep{GuptaWojtowytsch2024}.

Finally, the exact quadratic calculations distinguish transient acceleration from first-order statistical efficiency.  Momentum changes the transient roots, marginal variance, and nonzero-frequency spectrum, but the long-run covariance of the full average remains $H^{-1}\Omega H^{-1}$.  This distinction, together with the nonuniformity of the small-step expansion near unit momentum, is essential when selecting momentum for online inference.

\clearpage
\appendix

\section{Power inequalities and proof of Proposition~\ref{prop:base}}
\label{app:base}

For completeness, we reproduce the analytic inequalities underlying the ordinary-SGD module.  They are the part of the GMC argument that permits noninteger $p$ and infinite variance.

\begin{lemma}[Euclidean power remainders]
\label{lem:power}
For $x,y\in\R^d$:
\begin{enumerate}[label=\textup{(\roman*)},leftmargin=2.2em]
\item if $p\geq2$,
\begin{equation}
 \left|\abs{x+y}^p-\abs x^p-p\abs x^{p-2}\ip{x}{y}\right|
 \leq(\abs x+\abs y)^p-\abs x^p-p\abs x^{p-1}\abs y;
 \label{eq:power-large}
\end{equation}
\item if $1<p<2$,
\begin{equation}
 \abs{x+y}^p
 \leq\abs x^p+p\abs x^{p-2}\ip{x}{y}
 +2^{2-p}\abs y^p.
 \label{eq:power-small}
\end{equation}
\end{enumerate}
\end{lemma}

\begin{proof}
For $p\geq2$, rotate coordinates so that $x=\delta e_1$ and decompose $y=a\delta e_1+be_2$.  With $R=\abs y$ and $a\delta=Ru$, $u\in[-1,1]$, the centered remainder equals
\[
 f(u)=(\delta^2+2\delta Ru+R^2)^{p/2}-\delta^p-p\delta^{p-1}Ru.
\]
Differentiation shows that $f$ decreases and then increases, and its largest absolute value on $[-1,1]$ is attained at $u=1$.  This gives \eqref{eq:power-large}.

For $1<p<2$, after dividing by $\abs y^p$ when $x,y\neq0$, put $\omega=\abs x/\abs y$ and $\rho=\ip{x}{y}/(\abs x\abs y)$.  The normalized remainder is
\[
 \psi(\omega,\rho)
 =(\omega^2+1+2\omega\rho)^{p/2}
 -\omega^p-p\omega^{p-1}\rho.
\]
Maximizing over $\rho\in[-1,1]$ gives either the interior point $\rho=-1/(2\omega)$ or the endpoint $\rho=-1$.  If $\omega\geq1/2$, the interior maximum is $(p/2)\omega^{p-2}\leq p2^{1-p}\leq2^{2-p}$.  If $0<\omega<1/2$, the maximum is $(1-\omega)^p-\omega^p+p\omega^{p-1}$, bounded by $2^{2-p}$ using Proposition~1.8 of \citet{Pinelis2015}.  The constant $2^{2-p}$ is a convenient valid upper bound, not the best constant for a fixed $p$.  Multiplying back by $\abs y^p$ proves \eqref{eq:power-small}.
\end{proof}

\begin{proof}[Proof of Proposition~\ref{prop:base}]
Fix $\theta,\vartheta$, set
\[
 x=\theta-\vartheta,
 \qquad
 y=-\gamma\{G(\theta,X_0)-G(\vartheta,X_0)\}.
\]
When $p\geq2$, first justify the scalar expected-remainder step.  For fixed $a\geq0$, put
\[
 \phi_a(t)=(a+t^{1/p})^p-a^p-pa^{p-1}t^{1/p},\qquad t\geq0.
\]
For $t>0$, writing $u=t^{1/p}$ gives
$\phi_a'(t)=(1+a/u)^{p-1}-(a/u)^{p-1}$.
As $u$ increases this derivative decreases: the function
$r\mapsto(1+r)^{p-1}-r^{p-1}$ is nondecreasing for $p\geq2$.
Thus $\phi_a$ is concave, extends continuously at zero, and Jensen's inequality yields
\begin{equation}
 \E\{(a+|y|)^p-a^p-pa^{p-1}|y|\}
 \leq (a+\|y\|_p)^p-a^p-pa^{p-1}\|y\|_p.
 \label{eq:jensen-remainder}
\end{equation}
(The case $a=0$ is linear.)  This is the required justification; it is not obtained by simply substituting an $L^p$ norm in the negative linear term.  Lemma~\ref{lem:power}(i), mean monotonicity, and \eqref{eq:jensen-remainder} now give
\begin{align*}
 \E\abs{x+y}^p
 &\leq\abs x^p-p\mu\gamma\abs x^p\\
 &\quad+\E(\abs x+\abs y)^p-\abs x^p
 -p\abs x^{p-1}\E\abs y\\
 &\leq\left\{(1+\gamma L_p)^p-p\gamma L_p-p\mu\gamma\right\}\abs x^p.
\end{align*}
For $1<p<2$, Lemma~\ref{lem:power}(ii) yields
\begin{align*}
 \E\abs{x+y}^p
 &\leq\abs x^p-p\mu\gamma\abs x^p
 +2^{2-p}\gamma^pL_p^p\abs x^p.
\end{align*}
Taking $p$th roots gives \eqref{eq:base-contract} with \eqref{eq:qdef}.  The step-size characterizations follow by elementary monotonicity of the right-hand sides.
\end{proof}

\section{Additional details for the iterated-random-function construction}
\label{app:irf}

This appendix records a general lemma used implicitly in the main proofs.

\begin{lemma}[Backward construction in a product $L^p$ metric]
\label{lem:backward}
Let $\Phi_X:E\to E$ be an i.i.d. random recursion on a finite-dimensional normed space.  Suppose there is a complete metric $\mathfrak D_p$ on $L^p(E)$, equivalent to the componentwise $L^p$ norm, and an $r<1$ such that for any past-measurable $U,V$ independent of a fresh $X$,
\begin{equation}
 \mathfrak D_p\{\Phi_X(U),\Phi_X(V)\}
 \leq r\mathfrak D_p(U,V).
 \label{eq:abstract-contract}
\end{equation}
If $\mathfrak D_p(\Phi_X(z_0),z_0)<\infty$ for one deterministic anchor $z_0$, then there is a unique causal stationary solution with finite $p$th moment, and forward synchronous coupling decays at rate $r^k$.
\end{lemma}

\begin{proof}
Let
\[
 B_n(z_0)=\Phi_{X_0}\circ\Phi_{X_{-1}}\circ\cdots\circ
 \Phi_{X_{-n+1}}(z_0).
\]
The first $n$ maps in $B_{n+1}$ and $B_n$ are common, so \eqref{eq:abstract-contract} gives
\[
 \mathfrak D_p\{B_{n+1}(z_0),B_n(z_0)\}
 \leq r^n\mathfrak D_p\{\Phi_{X_{-n}}(z_0),z_0\}.
\]
The right-hand side has constant finite expectation and is summable in $n$.  Hence $B_n(z_0)$ is Cauchy in $L^p$ and converges to a causal limit $Z_0^\circ$.  Shift covariance gives a stationary sequence satisfying the recursion.  Starting the backward composition from any other deterministic point changes the result by at most $r^n$ times the initial distance, so the limit is anchor-independent.  If $Z^\circ$ and $\widetilde Z^\circ$ are two stationary solutions under common innovations, then
\[
 \mathfrak D_p(Z_0^\circ,\widetilde Z_0^\circ)
 \leq r^n\mathfrak D_p(Z_{-n}^\circ,\widetilde Z_{-n}^\circ)
 =r^n\mathfrak D_p(Z_0^\circ,\widetilde Z_0^\circ),
\]
forcing equality to zero.  The forward bound follows by iterating \eqref{eq:abstract-contract}.
\end{proof}

\section{Algebraic verification of the all-\texorpdfstring{$\beta$}{beta} energy}
\label{app:allbeta}

For reference, we collect the exact transformations behind Theorem~\ref{thm:allbeta}.  From \eqref{eq:diff-recursion},
\begin{align*}
 u_{k+1}
 &=x_{k+1}+\frac{\beta}{d_\beta}s_{k+1}\\
 &=x_k+\beta s_k-\gamma h_{k+1}
 +\frac{\beta}{d_\beta}(\beta s_k-\gamma h_{k+1})\\
 &=x_k+\frac{\beta}{d_\beta}s_k
 -\frac{\gamma}{d_\beta}h_{k+1}.
\end{align*}
Also
\[
 u_k=x_k+\frac{\beta}{d_\beta}s_k
 =y_k+\frac{\beta^2}{d_\beta}s_k.
\]
For arbitrary $a>0$,
\begin{align*}
 V_{k+1}(a)-V_k(a)
 &=-\frac{2\gamma}{d_\beta}\ip{u_k}{h_{k+1}}
 -2a\beta\gamma\ip{s_k}{h_{k+1}}\\
 &\quad+\gamma^2(d_\beta^{-2}+a)\abs{h_{k+1}}^2
 -a(1-\beta^2)\abs{s_k}^2.
\end{align*}
All coefficients in Theorem~\ref{thm:allbeta} and Proposition~\ref{prop:optimal-weight} follow by taking conditional expectations, inserting $u=y+\beta^2s/d_\beta$, and using \eqref{eq:mono-y}--\eqref{eq:h2}.

\section{Further quadratic formulas}
\label{app:quadratic}

For the scalar model, the lag-one correlation is
\begin{equation}
 \frac{\Cov(e_k,e_{k-1})}{\Var(e_k)}
 =\frac{(1+\beta)(1-\gamma h)}
 {1+\beta(1-\gamma h)}.
 \label{eq:lagone}
\end{equation}
This follows from $r_1=a v/(1-c)$ in the proof of Corollary~\ref{cor:variance}.  The spectral density is
\begin{equation}
 f_e(\omega)
 =\frac{\gamma^2\sigma^2}{2\pi}
 \left|1-(1+\beta)(1-\gamma h)e^{-i\omega}
 +\beta(1-\gamma h)e^{-2i\omega}\right|^{-2}.
 \label{eq:spectral}
\end{equation}
At frequency zero, $2\pi f_e(0)=\sigma^2/h^2$, in agreement with Theorem~\ref{thm:lrcov-quadratic}.  At nonzero frequencies, both $\gamma$ and $\beta$ alter the spectrum; the invariance is specifically a zero-frequency statement relevant to full iterate averaging.

\section{A positivity fact for the covariance operator}
\label{app:cone}

\begin{lemma}[Spectral radius of a positive covariance operator]
\label{lem:cone-perron}
Let $\mathcal T$ be a linear map on the real symmetric $m\times m$ matrices such that $C\succeq0$ implies $\mathcal T(C)\succeq0$.  Its spectral radius is a nonnegative real eigenvalue with a nonzero positive-semidefinite eigenmatrix.
\end{lemma}
\begin{proof}
For $\epsilon>0$, put $\mathcal T_\epsilon(C)=\mathcal T(C)+\epsilon\tr(C)I_m$.  This maps every nonzero positive-semidefinite matrix to a positive-definite matrix.  On the compact convex set $\{C\succeq0:\tr C=1\}$, the continuous map
\[
 C\longmapsto\frac{\mathcal T_\epsilon(C)}{\tr\{\mathcal T_\epsilon(C)\}}
\]
has a fixed point by the finite-dimensional fixed-point theorem.  Thus there are $P_\epsilon\succ0$, $\tr P_\epsilon=1$, and $\lambda_\epsilon>0$ with $\mathcal T_\epsilon(P_\epsilon)=\lambda_\epsilon P_\epsilon$.  For every symmetric $C$ there is a finite $a_C$ such that $-a_CP_\epsilon\preceq C\preceq a_CP_\epsilon$.  Positivity and iteration give
\[
 -a_C\lambda_\epsilon^nP_\epsilon\preceq\mathcal T_\epsilon^n(C)
 \preceq a_C\lambda_\epsilon^nP_\epsilon.
\]
Applying this to a finite basis bounds the operator norm of $\mathcal T_\epsilon^n$ by a fixed constant times $\lambda_\epsilon^n$.  Hence $\rho(\mathcal T_\epsilon)\leq\lambda_\epsilon$, and the reverse inequality follows from the eigenmatrix.  Along a sequence $\epsilon\downarrow0$, compactness gives a limit $P\succeq0$ with $\tr P=1$.  Continuity of the spectral radius for finite-dimensional matrices and of the map gives $\mathcal T(P)=\rho(\mathcal T)P$.  This proves the claim, including the case of zero spectral radius.
\end{proof}

\bibliographystyle{plainnat}
\bibliography{Nesterov_GMC_References_ChenWu2016}

\end{document}